\documentclass{article}

\usepackage[nonatbib,final]{neurips_2018}

\usepackage[utf8]{inputenc} 
\usepackage[T1]{fontenc}    
\usepackage{hyperref}       
\usepackage{url}            
\usepackage{booktabs}       
\usepackage{amsfonts}       
\usepackage{nicefrac}       
\usepackage{microtype}      
\usepackage{amsmath,amssymb,amsthm,amsfonts,bm,url,dsfont,amsthm,hyperref,color}
\usepackage{makecell}
\usepackage{multirow}
\usepackage{float}
\usepackage{algorithm}
\usepackage{graphicx}

\def\H{\mathcal{H}}
\def\calA{\mathcal{A}}
\def\calR{\mathcal{R}}
\def\calS{\mathcal{S}}
\def\calI{\mathcal{I}}
\def\S{\mathcal{S}}
\def\E{\mathbb{E}}
\def\1{\mathbf{1}}
\def\P{\mathbb{P}}
\def\R{\mathbb{R}}

\makeatletter
\newtheorem*{rep@theorem}{\rep@title}
\newcommand{\newreptheorem}[2]{%
\newenvironment{rep#1}[1]{%
 \def\rep@title{#2 \ref{##1}}%
 \begin{rep@theorem}}%
 {\end{rep@theorem}}}
\makeatother

\newtheorem{lemma}{Lemma}

\newtheorem{definition}{Definition}

\newtheorem{theorem}{Theorem}

\newtheorem{remark}{Remark}
\newreptheorem{theorem}{Theorem}
\newreptheorem{lemma}{Lemma}

\title{A Bandit Approach to Multiple Testing with False Discovery Control}

 \author{{Kevin Jamieson$^{\ast,\dagger}$, Lalit Jain$^\ast$} \\
 \texttt{\{jamieson,lalitj\}@cs.washington.edu} \\
 $^\ast$Paul G. Allen School of Computer Science \& Engineering,\\ University of Washington, Seattle, WA, and \\
 $^\dagger$Optimizely, San Francisco, CA
 }

\begin{document}

\maketitle

\begin{abstract}
 We propose an adaptive sampling approach for multiple testing which aims to maximize statistical power while ensuring anytime false discovery control. We consider $n$ distributions whose means are partitioned by whether they are below or equal to a baseline (nulls), versus above the baseline (actual positives). In addition, each distribution can be sequentially and repeatedly sampled. Inspired by the multi-armed bandit literature, we provide an algorithm that takes as few samples as possible to exceed a target true positive proportion (i.e. proportion of actual positives discovered) while giving anytime control of the false discovery proportion (nulls predicted as actual positives). Our sample complexity results match known information theoretic lower bounds and through simulations we show a substantial performance improvement over uniform sampling and an adaptive elimination style algorithm. Given the simplicity of the approach, and its sample efficiency, the method has promise for wide adoption in the biological sciences, clinical testing for drug discovery, and online A/B/n testing problems.
\end{abstract}

\section{Introduction}
Consider $n$ possible treatments, say, drugs in a clinical trial, where each treatment either has a positive expected effect relative to a baseline (actual positive), or no difference (null), with a goal of identifying as many actual positive treatments as possible.
If evaluating the $i$th trial results in a noisy outcome (e.g. due to variance in the actual measurement or just diversity in the population) then given a total measurement budget of $B$, it is standard practice to execute and average $B/n$ measurements of each treatment, and then output a set of predicted actual positives based on the measured effect sizes.
False alarms (i.e. nulls predicted as actual positives) are controlled by either controlling \emph{family-wise error rate (FWER)}, where one bounds the probability that at least one of the predictions is null, or \emph{false discovery rate (FDR)}, where one bounds the expected proportion of the number of predicted nulls to the number of predictions.
FDR is a weaker condition than FWER but is often used in favor of FWER because of its higher \emph{statistical power}: more actual positives are output as predictions using the same measurements.

In the pursuit of even greater statistical power, there has recently been increased interest in the biological sciences to reject the uniform allocation strategy of $B/n$ trials to the $n$ treatments in favor of an \emph{adaptive} allocation.
Adaptive allocations partition the budget $B$ into sequential rounds of measurements in which the measurements taken at one round inform which measurements are taken in the next \cite{hao2008drosophila,rocklin}.
Intuitively, if the effect size is relatively large for some treatment, fewer trials will be necessary to identify that treatment as an actual positive relative to the others, and that savings of measurements can be allocated towards treatments with smaller effect sizes to boost the signal.
However, both \cite{hao2008drosophila,rocklin} employed ad-hoc heuristics which may not only have sub-optimal statistical power, but also may even result in more false alarms than expected.
As another example, in the domain of A/B/n testing in online environments, the desire to understand and maximize click-through-rate across treatments (e.g., web-layouts, campaigns, etc.) has become ubiquitous across retail, social media, and headline optimization for the news.
And in this domain, the desire for statistically rigorous adaptive sampling methods with high statistical power are explicit \cite{johari2015always}.

In this paper we propose an adaptive measurement allocation scheme that achieves near-optimal statistical power subject to FWER or FDR false alarm control.
Perhaps surprisingly, we show that even if the treatment effect sizes of the actual positives are identical, adaptive measurement allocation can still substantially improve statistical power.
That is, more actual positives can be predicted using an adaptive allocation relative to the uniform allocation under the same false alarm control.

\subsection{Problem Statement}\label{sec:problem_statement}
Consider $n$ distributions (or arms) and a game where at each time $t$, the player chooses an arm $i \in [n] := \{1,\dots,n\}$ and immediately observes a reward $X_{i,t} \overset{iid}{\sim} \nu_i$ where $X_{i,t} \in [0,1]$\footnote{All results without modification apply to unbounded, sub-Gaussian random variables.} and $\E_{\nu_i}[X_{i,t}] = \mu_i$.
For a \emph{known} threshold $\mu_0$, define the sets\footnote{All results generalize to the case when $\H_0 = \{ i : \mu_i \leq \mu_0 \}$.}
\begin{align*}
\H_1 = \{ i \in [n]: \mu_i > \mu_0 \} \quad \text{ and } \quad \H_0 = \{ i \in [n] : \mu_i = \mu_0 \} = [n] \setminus \H_1 .
\end{align*}
The value of the means $\mu_i$ for $i\in[n]$ and the cardinality of $\H_1$ are \emph{unknown}.
The arms (treatments) in $\H_1$ have means greater than $\mu_0$ (positive effect) while those in $\H_0$ have means equal to $\mu_0$ (no effect over baseline).
At each time $t$, after the player plays an arm, she also outputs a set of indices $\calS_t \subseteq [n]$ that are interpreted as \emph{discoveries} or rejections of the null-hypothesis (that is, if $i \in \calS_t$ then the player believes $i \in \H_1$).
For as small a $\tau \in \mathbb{N}$ as possible, the goal is to have the number of true detections $|\calS_t \cap \H_1|$ be approximately $|\H_1|$ for all $t \geq \tau$, subject to the number of false alarms $|\calS_t \cap \H_0|$ being small uniformly over all times $t \in \mathbb{N}$.
We now formally define our notions of false alarm control and true discoveries.
\begin{definition}[False Discovery Rate, FDR-$\delta$]
Fix some $\delta \in (0,1)$.
We say an algorithm is FDR-$\delta$ if for all possible problem instances $(\{\nu_i\}_{i=1}^n, \mu_0)$ it satisfies $\displaystyle\E[\tfrac{|\calS_t \cap \H_0|}{|\calS_t|\vee 1}] \leq \delta$ for all $t \in \mathbb{N}$ simultaneously.
\end{definition}
\begin{definition}[Family-wise Error Rate, FWER-$\delta$]
Fix some $\delta \in (0,1)$.
We say an algorithm is FWER-$\delta$ if for all possible problem instances $(\{\nu_i\}_{i=1}^n, \mu_0)$ it satisfies $\P(\bigcup_{t=1}^\infty \{ \calS_t \cap \H_0 \neq \emptyset\} ) \leq \delta$.
\end{definition}
\noindent Note FWER-$\delta$ implies FDR-$\delta$, the former being a stronger condition than the latter.
Allowing a relatively small number of false discoveries is natural, especially if $|\H_1|$ is relatively large.
Because $\mu_0$ is known, there exist schemes that guarantee FDR-$\delta$ or FWER-$\delta$ even if the arm means $\mu_i$ and the cardinality of $\H_1$ are unknown (see Section~\ref{sec:false_alarm_control}).
It is also natural to relax the goal of identifying \emph{all} arms in $\H_1$ to simply identifying a \emph{large proportion} of them.
\begin{definition}[True Positive Rate, TPR-$\delta,\tau$]
Fix some $\delta \in (0,1)$.
We say an algorithm is TPR-$\delta,\tau$ on an instance $(\{\nu_i\}_{i=1}^n, \mu_0)$ if $\E[\frac{|\calS_t \cap \H_1|}{|\H_1|}] \geq 1-\delta$ for all  $t \geq \tau$.
\end{definition}
\begin{definition}[Family-wise Probability of Detection, FWPD-$\delta,\tau$]
Fix some $\delta \in (0,1)$.
We say an algorithm is FWPD-$\delta,\tau$ on an instance $(\{\nu_i\}_{i=1}^n, \mu_0)$ if $\P(\H_1 \subseteq \calS_t ) \geq 1- \delta$ for all  $t \geq \tau$.
\end{definition}
Note that FWPD-$\delta,\tau$ implies TPR-$\delta,\tau$, the former being a stronger condition than the latter.
Also note $\P( \bigcup_{t=1}^\infty \{ \calS_t \cap \H_0 \neq \emptyset\} ) \leq \delta$ and $\P( \H_1 \subseteq \calS_\tau) \geq 1-\delta$ together imply $\P(\H_1 = \calS_\tau) \geq 1-2\delta$.
We will see that it is possible to control the number of false discoveries $|\calS_t \cap \H_0|$ regardless of how the player selects arms to play.
It is the rate at which $\calS_t$ includes $\H_1$ that can be thought of as the statistical power of the algorithm, which we formalize as its \emph{sample complexity}:
\begin{definition}[Sample Complexity]
Fix some $\delta \in (0,1)$ and an algorithm $\calA$ that is FDR-$\delta$ (or FWER-$\delta$) over all possible problem instances.
Fix a particular problem instance $(\{\nu_i\}_{i=1}^n, \mu_0)$.
At each time $t \in \mathbb{N}$, $\calA$ chooses an arm $i \in [n]$ to obtain an observation from, and before proceeding to the next round outputs a set $\calS_t \subseteq [n]$.
The \emph{sample complexity} of $\calA$ on this instance is the smallest time $\tau \in \mathbb{N}$ such that $\calA$ is TPR-$\delta,\tau$ (or FWPD-$\delta,\tau$).
\end{definition}
The sample complexity and value of $\tau$ of an algorithm will depend on the particular instance $(\{\nu_i\}_{i=1}^n, \mu_0)$.
For example, if $\H_1 = \{i \in [n] : \mu_i =\mu_0 + \Delta\}$ and $\H_0 = [n] \setminus \H_1$, then we expect the sample complexity to increase as $\Delta$ decreases since at least $\Delta^{-2}$ samples are necessary to determine whether an arm has mean $\mu_0$ versus $\mu_0 + \Delta$. The next section will give explicit cases.

 \begin{remark}[Impossibility of stopping time] We emphasize that just as in the non-adaptive setting, at no time can an algorithm \emph{stop} and declare that it is TPR-$\delta,\tau$ or FWPD-$\delta,\tau$ for any finite $\tau \in \mathbb{N}$.
This is because there may be an arm in $\H_1$ with a mean infinitesimally close to $\mu_0$ but distinct such that no algorithm can determine whether it is in $\H_0$ or $\H_1$.
Thus, the algorithm must run indefinitely or until it is stopped externally.
However, using an anytime confidence bound (see Section~\ref{sec:algorithm}) one can always make statements like ``either $\H_1 \subseteq \calS_t$, or $\max_{i \in \H_1 \setminus \calS_t} \mu_i-\mu_0 \leq \epsilon$'' where the $\epsilon$ will depend on the width of the confidence interval.
\end{remark}

\subsection{Contributions and Informal Summary of Main Results}\label{sec:contributions}
\begin{table}
\begin{center}
\begin{tabular}{c | c c }
& \multicolumn{2}{c}{\textbf{False alarm control}}\\[2pt]
& \makecell{FDR-$\delta$ \\ $\max_t \E[\frac{|\calS_t \cap \H_0|}{|\calS_t| \vee 1}] \leq \delta$} & \makecell{FWER-$\delta$ \\ $\P( \bigcup_{t=1}^\infty \{ \calS_t \cap \H_0 \neq \emptyset\} ) \leq \delta$}  \\  \hline \\[-8pt]
\makecell{\textbf{Detection Probability}\\[4pt] TPR-$\delta,\tau$ \\ $\E[\frac{|\calS_\tau \cap \H_1|}{|\H_1|}] \geq 1-\delta$\\[8pt] } & \makecell{ {\small Theorem~\ref{thm:FDR_TPR}}\\[-0pt]  $n\Delta^{-2}$} &  \makecell{ {\small Theorem~\ref{thm:FWER_TPR}}\\[-0pt] $(n-k)\Delta^{-2} + k \Delta^{-2} \log(n-k)$} \\[16pt]
\makecell{FWPD-$\delta,\tau$ \\ $\P( \H_1 \subseteq \calS_\tau) \geq 1-\delta$} & \makecell{ {\small Theorem~\ref{thm:FDR_FWPD}}\\[-0pt] $\quad (n-k) \Delta^{-2} \log(k) + k \Delta^{-2} \quad$} & \makecell{ {\small Theorem~\ref{thm:FWER_FWPD}}\\[-0pt] $(n-k) \Delta^{-2}\log(k) + k \Delta^{-2}\log(n-k)$}
\end{tabular}
\end{center}
\caption{\small Informal summary of sample complexity results proved in this paper for $|\H_1|=k$, constant $\delta$ (e.g., $\delta=.05$) and $\Delta=\min_{i \in \H_1} \mu_i - \mu_0$. Uniform sampling across all settings requires at least $n \Delta^{-2} \log(n/k)$ samples, and in the FWER+FWPD setting requires $n \Delta^{-2}\log(n)$.
Constants and $\log\log$ factors are ignored.\label{tab:complexity_table}}
\vspace{-.25in}
\end{table}
In Section~\ref{sec:algorithm} we propose an algorithm that handles all four combinations of \{FDR-$\delta$, FWER-$\delta$\} and \{TPR-$\delta,\tau$, FWPD-$\delta,\tau$\}.
A reader familiar with the multi-armed bandit literature would expect an adaptive sampling algorithm to have a large advantage over uniform sampling when there is a large diversity in the means of $\H_1$ since larger means can be distinguished from $\mu_0$ with fewer samples.
However, one should note that to declare all of $\H_1$ as discoveries, one must sample every arm in $\H_0$ \emph{at least} as many times as the \emph{most sampled} arm in $\H_1$, otherwise they are statistically indistinguishable.
As discoveries are typically uncovering rare phenomenon, it is common to assume $|\H_1| = n^\beta$ for $\beta \in (0,1)$ \cite{castro2014adaptive,rabinovich2017optimal}, or $|\H_1| = o(n)$, but this implies that the number of samples taken from the arms in $\H_1$, regardless of how samples are allocated to those arms, will almost always be dwarfed by the number of samples allocated to those arms in $\H_0$ since there are $\Omega(n)$ of them.
This line of reasoning, in part, is what motivates us to give our sample complexity results in terms of the quantities that best describe the contributions from those arms in $\H_0$, namely, the cardinality $|\H_1| = n-|\H_0|$, the confidence parameter $\delta$ (e.g., $\delta=.05$), and the gap $\Delta := \min_{i \in \H_1} \mu_i - \mu_0$ between the means of the arms in $\H_0$ and the smallest mean in $\H_1$.
Reporting sample complexity results in terms of $\Delta$ also allows us to compare to known lower bounds in the literature \cite{2017arXiv170306222R,castro2014adaptive,malloy2014sequential,simchowitz2017simulator}.
Nevertheless, we do address the case where the means of $\H_1$ are varied in Theorem~\ref{thm:FDR_TPR}.


An informal summary of the sample complexity results proven in this work are found in Table~\ref{tab:complexity_table} for $|\H_1|=k$.
For the least strict setting of FDR+TPR, the upper-left quadrant of Table~\ref{tab:complexity_table} matches the lower bound of \cite{castro2014adaptive}, a sample complexity of just $\Delta^{-2}n$.
In this FDR+TPR setting (which requires the fewest samples of the four settings), uniform sampling which pulls each arm an equal number of times has a sample complexity of at least $n \Delta^{-2} \log(n/|\H_1|)$
(see Theorem \ref{thm:uniform_fdr_lower_bound} in Appendix~\ref{sec:succ-elim}), which exceeds all results in Table~\ref{tab:complexity_table} demonstrating the statistical power gained by adaptive sampling.
For the most strict setting of FWER+FWPD, the lower-right quadrant of Table~\ref{tab:complexity_table} matches the lower bounds of \cite{malloy2014sequential,kalyanakrishnan2012pac,simchowitz2017simulator}, a sample complexity of
 $(n-k) \Delta^{-2}\log(k) + k \Delta^{-2}\log(n-k)$.
 Uniform sampling in the FWER+FWPD setting has a sample complexity lower bounded by $n \Delta^{-2} \log(n)$ (see Theorem~\ref{thm:uniform_fwer_lower_bound} in Appendix~\ref{sec:succ-elim}).
The settings of FDR+FWPD and FWER+TPR are sandwiched between these results, and we are unaware of existing lower bounds for these settings.

All the results in Table~\ref{tab:complexity_table} are novel, and to the best of our knowledge are the first non-trivial sample complexity results for an adaptive algorithm in the \emph{fixed confidence} setting where a desired confidence $\delta$ is set, and the algorithm attempts to minimize the number of samples taken to meet the desired conditions.
We also derive tools that we believe may be useful outside this work: for always valid $p$-values (c.f. \cite{johari2015always,yang2017framework}) we show that FDR is controlled for all times using the Benjamini-Hochberg procedure \cite{benjamini1995controlling}  (see Lemma~\ref{lem:expected_FDR}), and also provide an anytime high probability bound on the false discovery proportion (see Lemma~\ref{lem:fdr_high_prob}).

Finally, as a direct consequence of the theoretical guarantees proven in this work and the empirical performance of the FDR+TPR variant of the algorithm on real data, an algorithm faithful to the theory was implemented and is in use in production at a leading A/B testing platform \cite{optimizely}.

\subsection{Related work}\label{sec:related_work}

Identifying arms with means above a threshold, or equivalently, multiple testing via rejecting null-hypotheses with small $p$-values, is an ubiquitous problem in the biological sciences.
In the standard setup, each arm is given an equal number of measurements (i.e., a uniform sampling strategy), a $p$-value $P_i$ is produced for each arm where $\P(P_i \leq x) \leq x$ for all $x \in (0,1]$ and $i \in \H_0$, and a procedure is then run on these $p$-values to declare small $p$-values as rejections of the null-hypothesis, or discoveries.
For a set of $p$-values $P_1 \leq P_2 \leq \dots \leq P_n$, the so-called Bonferroni selection rule selects $\calS_{BF} = \{ i : P_{i} \leq \delta/n \}$.
The fact that FWER control implies FDR control, $\E[ | \calS_{BF} \cap \H_0| ] \leq \P(\bigcup_{i \in \H_0} \{ P_{i} \leq \delta/n\} ) \leq \delta \frac{|\H_0|}{n} \leq \delta$,
suggests that greater statistical power (i.e. more discoveries) could be achieved with procedures designed specifically for FDR.
The BH procedure \cite{benjamini1995controlling} is one such procedure to control FDR and is widely used in practice (with its many extensions \cite{2017arXiv170306222R} and performance investigations \cite{rabinovich2017optimal}).
Recall that a uniform measurement strategy where every arm is sampled the same number of times requires
$n \Delta^{-2} \log(n/k)$ samples in the FDR+TPR setting, and $n \Delta^{-2}\log(n)$ samples in the FWER+FWPD setting (Theorems~\ref{thm:uniform_fdr_lower_bound} and \ref{thm:uniform_fwer_lower_bound} in Appendix~\ref{sec:succ-elim}), which can be substantially worse than our adaptive procedure (see Table~\ref{tab:complexity_table}).

Adaptive sequential testing has been previously addressed in the \emph{fixed budget} setting: the procedure takes a sampling budget as input, and the guarantee states that if the given budget is larger than a problem dependent constant, the procedure drives the error probability to zero and the detection probability to one.
One of the first methods called \emph{distilled sensing} \cite{haupt2011distilled} assumed that arms from $\H_0$ were Gaussian with mean at most $\mu_0$, and successively discarded arms after repeated sampling by thresholding at $\mu_0$--at most the median of the null distribution--thereby discarding about half the nulls at each round.
The procedure made guarantees about FDR and TPR, which were later shown to be nearly optimal \cite{castro2014adaptive}.
Specifically, \cite[Corollary 4.2]{castro2014adaptive} implies that any procedure with $\max\{FDR+(1-TPR)\} \leq \delta$ requires a budget of at least $\Delta^{-2} n \log(1/\delta)$, which is consistent with our work.
Later, another thresholding algorithm for the fixed budget setting addressed the FWER and FWPD metrics \cite{malloy2014sequential}.
In particular, if their procedure is given a budget exceeding $(n-|\H_1|) \Delta^{-2}\log(|\H_1|) + |\H_1| \Delta^{-2} \log(n-|\H_1|)$ then the FWER is driven to zero, and the FWPD is driven to one.
By appealing to the optimality properties of the SPRT (which knows the distributions precisely) it was argued that this is optimal.
These previous works mostly focused on the asymptotic regime as $n \rightarrow \infty$ and $|\H_1| = o(n)$.

Our paper, in contrast to these previous works considers the \emph{fixed confidence} setting: the procedure takes a desired FDR (or FWER) and TPR (or FWPD) and aims to minimize the number of samples taken before these constraints are met. To the best of our knowledge, our paper is the first to propose a scheme for this problem in the fixed confidence regime with near-optimal sample complexity guarantees.

A related line of work is the threshold bandit problem, where all the means of $\H_1$ are assumed to be strictly above a given threshold, and the means of $\H_0$ are assumed to be strictly below the threshold \cite{locatelli2016optimal,kano2017good}. To identify this partition, each arm must be pulled a number of times inversely proportional to the square of its deviation from the threshold. This contrasts with our work, where the majority of arms may have means \emph{equal} to the threshold and the goal is to identify arms with means greater than the threshold subject to discovery constraints. If the arms in $\H_0$ are assumed to be strictly below the threshold it is possible to declare arms as in $\H_0$. In our setting we can only ever determine that an arm is in $\H_1$ and not $\H_0$, but it is impossible to detect that an arm is in $\H_0$ and not in $\H_1$.

Note that the problem considered in this paper is very related to the top-$k$ identification problem where the objective is to identify the unique $k$ arms with the highest means with high probability \cite{chen2017nearly,kalyanakrishnan2012pac,simchowitz2017simulator}.
Indeed, if we knew $|\H_1|$, then our FWER+FWPD setting is equivalent to the top-$k$ problem with $k=|\H_1|$. Lower bounds derived for the top-$k$ problem assume the algorithm has knowledge of the values of the means, just not their indices \cite{chen2017nearly, simchowitz2017simulator}. Thus, these lower bounds also apply to our setting and are what are referenced in Section \ref{sec:contributions}.


As pointed out by \cite{locatelli2016optimal}, both our setting and the threshold bandit problem can be posed as a combinatorial bandits problem as studied in \cite{chen2014combinatorial,cao2017disagreement}, but such generality leads to unnecessary $\log$ factors.
The techniques used in this work aim to reduce extraneous $\log$ factors, a topic of recent interest in the top-$1$ and top-$k$ arm identification problem \cite{even2006action,karnin2013almost,jamieson2014lil,chen2017towards,chen2017nearly,simchowitz2017simulator}.
While these works are most similar to exact identification (FWER+FWPD), there also exist examples of \emph{approximate} top-$k$ where the objective is to find any $k$ means that are each within $\epsilon$ of the best $k$ means \cite{kalyanakrishnan2012pac}.
Approximate recovery is also studied in a ranking context with a symmetric difference metric \cite{heckel2018approximaterank} which is more similar to the FDR and TPR setting, but neither this nor that work subsumes one another.

Finally, maximizing the number of discoveries subject to a FDR constraint has been studied in a sequential setting in the context of A/B testing with uniform sampling \cite{johari2015always}.
This work popularized the concept of an always valid $p$-value that we employ here (see Section~\ref{sec:algorithm}).
The work of \cite{yang2017framework} controls FDR over a \emph{sequence} of independent bandit problems that each outputs at most one discovery.
While \cite{yang2017framework} shares much of the same vocabulary as this paper, the problem settings are very different.

\section{Algorithm and Discussion}\label{sec:algorithm}

\begin{algorithm}[t]
 	\textbf{Input:} Threshold $\mu_0$, confidence $\delta\in (0, e^{-1}]$, confidence interval $\phi(\cdot,\cdot)$\\
 	\textbf{Initialize:} Pull each arm $i \in [n]$ once and let $T_i(t)$ denote the number of times arm $i$ has been pulled up to time $t$. Set $\S_{n+1} = \emptyset$, $\calR_{n+1} = \emptyset$, and\\
 	\textbf{If } TPR\\
 	\hspace*{.25in}$\xi_t=1$, \hspace{.2in} and \hspace{.2in} $\nu_t=1 \quad \forall t$\\
 	\textbf{Else if } FWPD\\
 	\hspace*{.25in}$\xi_t=\max\{ 2|\S_t|, \tfrac{5}{3(1-4\delta)} \log(1/\delta) \}$, \hspace{.2in} and \hspace{.2in} $\nu_t=\max\{ |\S_t|, 1\} \quad \forall t$ \\[4pt]
	\textbf{For} $t = n+1,n+2,\dots$ \\
	\hspace*{.25in}\textbf{Pull arm} $\displaystyle I_t = \arg\max_{i \in [n] \setminus \S_t } \widehat{\mu}_{i,T_i(t)} + \phi(T_i(t),\tfrac{\delta}{\xi_t})$, \\
	\hspace*{.25in} \textbf{Apply} Benjamini-Hochberg \cite{benjamini1995controlling} selection at level $\delta' = \tfrac{\delta}{6.4\log(36/\delta)}$ to obtain $\delta$ FDR-controlled set $\calS_t$:\\[2pt]
	\hspace*{.5in}$s(k) = \{ i \in [n] : \widehat{\mu}_{i,T_i(t)} - \phi(T_i(t),\delta' \tfrac{k}{n}) \geq \mu_0 \}$, $\forall k \in [n]$ \\
	\hspace*{.5in}$\displaystyle\S_{t+1} = s(\widehat{k}) \text{ where }\widehat{k} = \max \{ k \in [n]: |s(k)| \geq k \}$ (if $\not\exists \widehat{k}$ set $\S_{t+1} = \S_t$)\\[4pt] 
	\hspace*{.25in}\textbf{If } FWER and $\S_t \neq \emptyset$:\\
	\hspace*{.5in}\textbf{Pull arm} $\displaystyle J_t = \arg\max_{i \in \S_t\setminus \calR_t} \widehat{\mu}_{i,T_i(t)} + \phi(T_i(t),\tfrac{\delta}{\nu_t})$ \\
	\hspace*{.5in} \textbf{Apply} Bonferroni-like selection to obtain FWER-controlled set $\calR_t$:\\[2pt]
	\hspace*{.75in}$\chi_t = n - (1-2\delta'(1+4\delta')) |\S_t| + \tfrac{4(1+4\delta')}{3}\log(5\log_2(n/\delta')/\delta')$\\
	\hspace*{.75in}$\displaystyle\calR_{t+1} = \calR_t \cup \{ i \in \S_t : \widehat{\mu}_{i,T_i(t)} - \phi(T_i(t), \tfrac{\delta}{\chi_t}) \geq \mu_0 \}$ \\ 
	\caption{\small An algorithm for identifying arms with means above a threshold $\mu_0$ using as few samples as possible subject to false alarm and true discovery conditions.
	The set $\calS_t$ is designed to control FDR at level $\delta$.
	The set $\calR_t$ is designed to control FWER at level $\delta$. \label{alg:FDRMAB}}
\end{algorithm}

Throughout, we will assume the existence of an \emph{anytime confidence interval}.
Namely, if $\widehat{\mu}_{i,t}$ denotes the empirical mean of the first $t$ bounded i.i.d. rewards in $[0,1]$ from arm $i$, then for any $\delta \in (0,1)$ we assume the existence of a function $\phi$ such that for any $\delta$ we have
$\P\left( \bigcap_{t=1}^\infty \{ |\widehat{\mu}_{i,t} - \mu_i| \leq \phi(t,\delta) \}  \right) \geq 1-\delta$.
We assume that $\phi( t , \delta )$ is non-increasing in its second argument and that there exists an absolute constant $c_\phi$ such that $\phi( t , \delta ) \leq \sqrt{\frac{ c_\phi \log( \log_2(2 t)/\delta)}{t}}$.
It suffices to define $\phi$ with this upper bound with $c_\phi=4$ but there are much sharper known bounds that should be used in practice (e.g., they may take empirical variance into account), see \cite{jamieson2014lil,Balsubramani2014,kaufmann2016complexity,tanczosnowak2017}.
Anytime bounds constructed with such a $\phi(t,\delta)$ are known to be tight in the sense that $\P(\bigcup_{t=1}^\infty \{|\widehat{\mu}_{i,t} - \mu_i| \geq \phi(t,\delta)\}) \leq \delta$ and that there exists an absolute constant $h \in (0,1)$ such that $\P(\{|\widehat{\mu}_{i,t} - \mu_i| \geq h \, \phi(t,\delta)\text{ for infinitely many $t \in \mathbb{N}$}\}) = 1$ by the Law of the Iterated Logarithm \cite{HartmanWintnerLIL}.

Consider Algorithm~\ref{alg:FDRMAB}.
Before entering the for loop, time-dependent variables $\xi_t$ and $\nu_t$ are defined that should be updated at each time $t$ for different settings.
If just FDR control is desired, the algorithm merely loops over the three lines following the for loop, pulling the arm $I_t$ not in $\calS_t$ that has the highest upper confidence bound; such strategies are common for pure-exploration problems \cite{jamieson2014lil,yang2017framework}.
But if FWER control is desired then at most one additional arm $J_t$ is pulled per round to provide an extra layer of filtering and evidence before an arm is added to $\calR_t$.
Below we describe the main elements of the algorithm and along the way sketch out the main arguments of the analysis, shedding light on the constants $\xi_t$ and $\nu_t$.

\subsection{False alarm control}\label{sec:false_alarm_control}
\textbf{$\calS_t$ is FDR-controlled.}
In addition to its use as a confident bound, we can also use $\phi(t,\delta)$ to construct:
\begin{align}\label{eqn:anytime_pvalue}
P_{i,t} := \sup \{ \alpha \in (0,1] : \widehat{\mu}_{i,t}-\mu_0 \leq \phi(t,\alpha) \} \leq \log_2(2t) \exp(-t (\widehat{\mu}_{i,t}-\mu_0)^2/c_\phi).
\end{align}
Proposition~1 of \cite{yang2017framework} (and the proof of our Lemma~\ref{lem:expected_FDR}) shows that if $i \in \H_0$ so that $\mu_i = \mu_0$ then $P_{i,t}$ is an \emph{anytime, sub-uniformly distributed $p$-value} in the sense that $\P( \bigcup_{t=1}^\infty \{ P_{i,t} \leq x \} ) \leq x$.
Sequences that have this property are sometimes referred to as \emph{always-valid} $p$-values \cite{johari2015always}.
Note that if $i \in \H_1$ so that $\mu_i > \mu_0$, we would intuitively expect the sequence $\{P_{i,t}\}_{t=1}^\infty$ to be point-wise smaller than if $\mu_i = \mu_0$ by the property that $\phi(\cdot,\cdot)$ is non-increasing in its second argument.
This leads to the intuitive rule to reject the null-hypothesis (i.e., declare $i \notin \H_0$) for those arms $i \in [n]$ where $P_{i,t}$ is very small.
The Benjamini-Hochberg (BH) procedure introduced in \cite{benjamini1995controlling} proceeds by first sorting the $p$-values so that $P_{(1),T_{(1)}(t)} \leq P_{(2),T_{(2)}(t)} \leq \dots \leq P_{(n),T_{(n)}(t)}$, then defines $\widehat{k} = \max\{ k : P_{(k),T_{(k)}(t)} \leq \delta \tfrac{k}{n} \}$, and sets $\calS_{BH} = \{ i : P_{i,T_i(t)} \leq \delta \tfrac{\widehat{k}}{n} \}$.
Note that this procedure is identical to defining sets
\begin{align*}
s(k) = \{ i : P_{i,T_i(t)} \leq \delta \tfrac{k}{n} \} = \{ i : \widehat{\mu}_{i,T_i(t)} - \phi(T_i(t),\delta\tfrac{k}{n}) \geq \mu_0 \},
\end{align*}
setting $\widehat{k} = \max\{ k : |s(k)| \geq k \}$, and $\calS_{BH} = s(\widehat{k})$, which is exactly the set $\calS_t = \calS_{BH}$ in Algorithm~\ref{alg:FDRMAB}.
Thus, $\calS_t$ in Algorithm~\ref{alg:FDRMAB} is equivalent to applying the BH procedure at a level $O(\delta/\log(1/\delta))$ to the anytime $p$-values of \eqref{eqn:anytime_pvalue}. We now discuss the extra logarithmic factor.

Because the algorithm is pulling arms sequentially, some dependence between the $p$-values may be introduced. Because the anytime $p$-values are not independent, the BH procedure at level $\delta$ does not directly guarantee FDR-control at level $\delta$.
However, it has been shown \cite{benjamini2001control} that for even arbitrarily dependent $p$-values the BH procedure at level $\delta$ controls FDR at level $\delta \log(n)$ (and that it is nearly tight).
Similarly, the following theorem, which may be of independent interest,
is a significant improvement when applied to our setting.

\begin{theorem}\label{thm:newFDRcontrol}
Fix $\delta \in (0,e^{-1})$.
Let $p_1,\dots,p_n$ be random variables such that $\{ p_i \}_{i\in \H_0}$ are independent and sub-uniformly distributed so that $\max_{i \in \H_0} \P(p_i \leq x ) \leq x$. For any $k \in \{0,1,\dots,n\}$, let $R_k := \{ i : p_i \leq \delta \frac{k}{n} \}$ and $\widehat{FDP}(R_k) := \tfrac{\max_{p_i \in R_k} p_i }{|R_k| \vee 1}$.
\begin{align*}
    \E\left[ \max_{k : \widehat{FDP}(R_k) \leq \delta} FDP(R_k) \right]
    &\leq \frac{|\H_0|\delta}{n} \left( 2  \log( \tfrac{2n}{|\H_0|\delta }) + \log( 8 e^5\log( \tfrac{8n}{|\H_0| \delta} ) ) \right) \\
    &\leq 4 \delta \log( 9/\delta)
\end{align*}
In other words, any procedure that chooses a set $\{ i : p_i \leq \frac{\delta k}{n} \}$   satisfying $|\{ i : p_i \leq \frac{\delta k}{n} \}| \geq k$ is FDR controlled at level $O(\delta \log(1/\delta))$.
\end{theorem}

Recall, if $\widehat{k} = \max\{ k : \widehat{FDP}(R_k) \leq \delta \}$ then $\E[ FDP(R_{\widehat{k}})] \leq \delta$ by the standard BH result.
When running the algorithm we recommend using BH at level $\delta$, not level  $O(\delta/\log(1/\delta))$.
As $T_i$ gets very large, $P_{i,T_i(t)} \rightarrow P_{i,*}$ and we know that if BH is run on $P_{i,*}$ at level $\delta$ then FDR would be controlled at level $\delta$. We believe this inflation to be somewhat of an artifact of our proofs.


\textbf{$\calR_t$ is FWER-controlled.}
A core obstacle in our analysis is the fact that we don't know the cardinality of $\H_1$.
If we did know $|\H_1|$ (and equivalently know $|\H_0| = n - |\H_1|$) then a FWER+FWPD algorithm is equivalent to the so-called top-$k$ multi-armed bandit problem \cite{kalyanakrishnan2012pac,simchowitz2017simulator} and controlling FWER would be relatively simple {using a Bonferroni correction:
\begin{align*}
\P\Big( \bigcup_{i \in \H_0} \cup_{t=1}^\infty \{ \widehat{\mu}_{i,t} - \phi(t, \tfrac{\delta}{n-|\H_1|}) \geq \mu_0\}\Big) \leq \sum_{i \in \H_0} \P\left(\cup_{t=1}^\infty \{ \widehat{\mu}_{i,t} - \phi(t, \tfrac{\delta}{|\H_0|}) \geq \mu_0\} \right)\leq |\H_0| \tfrac{\delta}{|\H_0|}
\end{align*}
which implies FWER-$\delta$.
Comparing the first expression immediately above to the definition of $\calR_t$ in the algorithm, it is clear our strategy is to use $|\calS_t|$ as a surrogate for $|\H_1|$.
Note that we could use the bound $|\H_0| = n - |\H_1| \leq n$ to guarantee FWER-$\delta$, but this could be very loose and induce an $n\log(n)$ sample complexity.
Using $|\calS_t|$ as a surrogate for $|\H_1|$ in $\calR_t$ is intuitive because by the FDR guarantee, we know $|\H_1| \geq \E[|\calS_t \cap \H_1|] = \E[|\calS_t|] - \E[|\calS_t \cap \H_0|] \geq (1-\delta) \E[ |\calS_t|]$, implying that $|\H_0| = n - |\H_1| \leq n - (1-\delta) \E[|\calS_t|]$ which may be much tighter than $n$ if $\E[|\calS_t|] \rightarrow |\H_1|$.
Because we only know $|\calS_t|$ and not its expectation, the extra factors in the surrogate expression used in $\calR_t$ are used to ensure correctness with high-probability (see Lemma~\ref{lem:fwer}).

\subsection{Sampling strategies to boost statistical power}
The above discussion about controlling false alarms for $\calS_t$ and $\calR_t$ holds for \emph{any} choice of arms $I_t$ and $J_t$ that may be pulled at time $t$.
Thus, $I_t$ and $J_t$ are chosen in order to minimize the amount of time necessary to add arms into $\calS_t$ and $\calR_t$, respectively, and optimize the sample complexity.

\textbf{TPR-$\delta,\tau$ setting} implies $\xi_t=\nu_t=1$.
Define the random set $\calI = \{ i \in \H_1 : \widehat{\mu}_{i,T_i(t)} + \phi(T_i(t),\delta) \geq \mu_i \ \ \forall t \in \mathbb{N} \}$.
Because $\phi$ is an anytime confidence bound, $\E\left[ \left| \calI \right| \right] \geq (1-\delta) |\H_1|$.
If $\Delta = \min_{i \in \H_1} \mu_i - \mu_0$, then $\min_{i \in \calI} \mu_i \geq \mu_0 + \Delta$ and we claim that with probability at least $1-O(\delta)$ (Section~\ref{sec:FDR_TPR_proof})
\begin{align*}
\textstyle\sum_{t=1}^\infty \1\{I_t \in \H_0, \calI \not\subseteq \calS_t \} &\leq \textstyle\sum_{t=1}^\infty \1\{I_t \in \H_0, \widehat{\mu}_{I_t,T_{I_t}(t)} + \phi(T_{I_t}(t),\delta) \geq \mu_0 + \Delta \}\\
&\leq c|\H_0| \Delta^{-2}\log(\log(\Delta^{-2}/\delta).
\end{align*}
Thus once this number of samples has been taken, either $\calI \subseteq \calS_t$, or arms in $\calI$ will be repeatedly sampled until they are added to $\calS_t$ since each arm $i \in \calI$ has its upper confidence bound larger than those arms in $\H_0$ by definition.
It is clear that an arm in $\H_1$ that is repeatedly sampled will eventually be added to $\calS_t$ since its anytime $p$-value of \eqref{eqn:anytime_pvalue} approaches $0$ at an exponential rate as it is pulled, and BH selects for low $p$-values.
A similar argument holds for $J_t$ and adding arms to $\calR_t$.

\begin{remark}
While the main objective of Algorithm~\ref{alg:FDRMAB} is to identify all arms with means above a given threshold, we note that prior to adding an arm to $\calS_t$ in the TPR setting (i.e., when $\xi_t=1$) Algorithm~\ref{alg:FDRMAB} behaves identically to the nearly optimal best-arm identification algorithm lil'UCB of \cite{jamieson2014lil}.
Thus, whether the goal is best-arm identification or to identify all arms with means above a certain threshold, Algorithm~\ref{alg:FDRMAB} is applicable.
\end{remark}

\textbf{FWPD-$\delta,\tau$ setting} is more delicate and uses inflated values of $\xi_t$ and $\nu_t$.
This time, we must ensure that $\{\H_1 \not\subseteq \calS_t\} \implies \max_{i \in \H_1 \cap S_t^c} \widehat{\mu}_{i,T_i(t)} + \phi(T_i(t),\delta) \geq \min_{i \in \H_1 \cap \calS_t^c} \mu_i \geq \mu_0 + \Delta$.
Because then we could argue that either $\H_1 \subset \calS_t$, or only arms in $\H_1$ are sampled until they are added to $\calS_t$ (mirroring the TPR argument).
As in the FWER setting above, if we knew the value of $|\H_1|$ the we could set $\xi_t \geq |\H_1|$ to observe that
\begin{align*}
\textstyle\P( \bigcup_{i \in \H_1} \cup_{t=1}^\infty \{ \widehat{\mu}_{i,t} + \phi(t, \tfrac{\delta}{\xi_t}) < \mu_i\} ) \leq \sum_{i \in \H_1} \P\left(\cup_{t=1}^\infty \{ \widehat{\mu}_{i,t} + \phi(t, \tfrac{\delta}{\xi_t}) < \mu_i\} \right)\leq |\H_1| \tfrac{\delta}{\xi_t}
\end{align*}
which is less than $\delta$, to guarantee such a condition.
But we don't know $|\H_1|$ so we use $|\calS_t|$ as a surrogate, resulting in the inflated definitions of $\xi_t$ and $\nu_t$ relative to the TPR setting.
The key argument is that either $\calI \not\subseteq \calS_t$ so that $\max_{i \in \calI \cap \calS_t^c} \widehat{\mu}_{i,T_i(t)} + \phi(T_i(t),\tfrac{\delta}{\xi_t}) \geq \mu_0+\Delta$ by the definition of $\calI$ (since $\xi_t \geq 1$), or $\calI \subset \calS_t$ and $|\calS_t| \geq \frac{1}{2} |\H_1|$ with high probability which implies $\xi_t = \max\{ 2|\S_t|, \tfrac{5}{3(1-4\delta)} \log(1/\delta) \} \geq |\H_1|$ and the union bound of the display above holds.

\section{Main Results}\label{sec:main_results}
In what follows, we say $f \lesssim g$ if there exists a $c >0$ that is independent of all problem parameters and $f \leq c g$.
The theorems provide an upper bound on the sample complexity $\tau \in \mathbb{N}$ as defined in Section~\ref{sec:problem_statement} for TPR-$\delta,\tau$ or FWER-$\delta,\tau$ that holds with probability at least $1-c \delta$ for different values of $c$\footnote{
Each theorem relies on different events holding with high probability, and consequently a different $c$ for each. To have $c=1$ for each of the four settings, we would have had to define different constants in the algorithm for each setting. We hope the reader forgives us for this attempt at minimizing clutter.}.
We begin with the least restrictive setting, resulting in the smallest sample complexity of all the results presented in this work.
Note the slight generalization in the below theorem where the means of $\H_0$ are assumed to be no greater than $\mu_0$.
\begin{theorem}[FDR, TPR]\label{thm:FDR_TPR}
Let $\H_1 = \{ i \in [n]: \mu_i > \mu_0\}$, $\H_0 = \{ i \in [n]: \mu_i \leq \mu_0 \}$. Define $\Delta_i = \mu_i - \mu_0$ for $i \in \H_1$, $\Delta = \min_{i \in \H_1} \Delta_i$, and $\Delta_i=\min_{j \in \H_1} \mu_j - \mu_i = \Delta + (\mu_0 - \mu_i)$ for $i\in\H_0$.
For all $t \in \mathbb{N}$ we have $\E[\frac{|\calS_t \cap \H_0|}{|\calS_t| \vee 1}] \leq \delta$.
Moreover, with probability at least $1-2\delta$ there exists a $T$ such that
\begin{align*}
\textstyle T \lesssim \min\big\{&n \Delta^{-2} \log( \log(\Delta^{-2})/\delta), \\
&\textstyle\sum_{i \in \H_0} \Delta_i^{-2} \log( \log(\Delta_i^{-2})/\delta) + \textstyle\sum_{i \in \H_1} \Delta_i^{-2}  \log( n \log(\Delta_i^{-2})/\delta)\big\}
\end{align*}
and $\E[\frac{|\calS_t \cap \H_1|}{|\H_1|}] \geq 1-\delta$ for all $t \geq T$.
Neither argument of the minimum follows from the other.
\end{theorem}
If the means of $\H_1$ are very diverse so that $\max_{i \in \H_1} \mu_i-\mu_0 \gg \min_{i \in \H_1} \mu_i-\mu_0$ then the second argument of the min in Theorem~\ref{thm:FDR_TPR} can be tighter than the first.
But as discussed above, this advantage is inconsequential if $|\H_1| = o(n)$.
The remaining theorems are given in terms of just $\Delta$.
The $\log\log(\Delta^{-2})$ dependence is due to inverting the $\phi$ confidence interval and is unavoidable on at least one arm when $\Delta$ is unknown a priori due to the law of the iterated logarithm \cite{HartmanWintnerLIL,jamieson2014lil,chen2017towards}.

Informally, Theorem~\ref{thm:FDR_TPR} states that if just most true detections suffice while not making too many mistakes, then $O(n)$ samples suffice.
The first argument of the min is known to be tight in a minimax sense up to doubly logarithmic factors due to the lower bound of \cite{castro2014adaptive}.
As a consequence of this work, an algorithm inspired by Algorithm~\ref{alg:FDRMAB} in this setting is now in production at one of the largest A/B testing platforms on the web.
The full proof of Theorem~\ref{thm:FDR_TPR} (and all others) is given in the Appendix due to space.
\begin{theorem}[FDR, FWPD]\label{thm:FDR_FWPD}
For all $t \in \mathbb{N}$ we have $\E[\frac{|\calS_t \cap \H_0|}{|\calS_t| \vee 1}] \leq \delta$.
Moreover, with probability at least $1-5\delta$, there exists a $T$ such that
\begin{align*}
T \lesssim  (n-|\H_1|) \Delta^{-2} \log (  \max\{|\H_1|,& \log\log(n/\delta)\} \log(\Delta^{-2})/ \delta)  + |\H_1| \Delta^{-2} \log ( \log(\Delta^{-2})/ \delta)
\end{align*}
and $\H_1 \subseteq \calS_t$ for all $t \geq T$.
\end{theorem}
Here $T$ roughly scales like $(n-|\H_1|) \max\{\log(|\H_1|), \log\log\log(n/\delta)\}  + |\H_1|$ where the $\log\log\log(n/\delta)$ term comes from a high probability bound on the false discovery proportion for anytime $p$-values (in contrast to just expectation) in Lemma~\ref{lem:fdr_high_prob} that may be of independent interest.
While negligible for all practical purposes, it appears unnatural and we suspect that this is an artifact of our analysis.
We note that if $|\H_1| = \Omega(\log(n))$ then the sample complexity sheds this awkwardness\footnote{In the asymptotic $n$ regime, it is common to study the case when $|\H_1| = n^\beta$ for $\beta \in (0,1)$ \cite{castro2014adaptive,haupt2011distilled}.}.

The next two theorems are concerned with controlling FWER on the set $\calR_t$ and determining how long it takes before the claimed detection conditions are satisfied on the set $\calR_t$.
Note we still have that FDR is controlled on the set $\mathcal{S}_t$ but now this set feeds into $\calR_t$.

\begin{theorem}[FWER, FWPD]\label{thm:FWER_FWPD}
For all $t$ we have $\E[\frac{|\calS_t \cap \H_0|}{|\calS_t| \vee 1}] \leq \delta$.
Moreover, with probability at least $1-6\delta$, we have $\H_0 \cap \calR_t = \emptyset$ for all $t \in \mathbb{N}$ and there exists a $T$ such that
\begin{align*}
T \lesssim&  (n-|\H_1|) \Delta^{-2} \log (  \max\{|\H_1|, \log\log(n/\delta)\} \log(\Delta^{-2}) / \delta) \\
&+ |\H_1| \Delta^{-2} \log ( \max\{n-(1-2\delta(1+4\delta))|\H_1|, \log\log(n/\delta)\}\log(\Delta^{-2}) / \delta)
\end{align*}
and $\H_1 \subseteq \calR_t$ for all $t \geq T$.
Note, together this implies $\H_1 = \calR_t$ for all $t \geq T$.
\end{theorem}
Theorem~\ref{thm:FWER_FWPD} has the strongest conditions, and therefore the largest sample complexity.
Ignoring $\log\log\log(n)$ factors, $T$ roughly scales as $(n-|\H_1|) \log(|\H_1|) + |\H_1| \log(n-(1-2\delta(1+4\delta))|\H_1|)$.
Inspecting the top-k lower bound of \cite{simchowitz2017simulator} where the arms' means in $\H_1$ are equal to $\mu_0 + \Delta$, the arms' means in $\H_0$ are equal to $\mu_0$, and the algorithm has knowledge of the cardinality of $\H_1$, a necessary sample complexity of $(n-|\H_1|)\log(|\H_1|) + |\H_1| \log(n-|\H_1|)$ is given.
It is not clear whether this small difference of $\log(n-(1-2\delta(1+4\delta)) |\H_1|)$ versus $\log(n-|\H_1|)$ is an artifact of our analysis, or a fundamental limitation when the cardinality $|\H_1|$ is unknown.
We now state our final theorem.

\begin{theorem}[FWER, TPR]\label{thm:FWER_TPR}
For all $t$ we have $\E[\frac{|\calS_t \cap \H_0|}{|\calS_t| \vee 1}] \leq \delta$.
Moreover, with probability at least $1-7\delta$ we have $\H_0 \cap \calR_t = \emptyset$ for all $t \in \mathbb{N}$ and there exists a $T$ such that
\begin{align*}
T \lesssim&  (n-|\H_1|) \Delta^{-2} \log (  \log(\Delta^{-2})/ \delta) \\
&+ |\H_1| \Delta^{-2} \log ( \max\{n-(1-\eta)|\H_1|, \log\log(n\log(1/\delta)/\delta)\}\log(\Delta^{-2}) / \delta)
\end{align*}
and $\E[\frac{|\calR_t \cap \H_1|}{|\H_1|}] \geq 1-\delta$ for all $t \geq T$, where $\eta=( 1 - 3\delta- \sqrt{2\delta \log(1/\delta)/|\H_1|})$.
\end{theorem}

\section{Experiments}
The distribution of each arm equals $\nu_i = \mathcal{N}(\mu_i,1)$ where $\mu_i = \mu_0 = 0$ if $i \in \H_0$, and $\mu_i>0$ if $i \in \H_1$.
We consider three algorithms: $i$) uniform allocation with anytime BH selection as done in Algorithm~1, $ii$) successive elimination (SE) (see Appendix~\ref{sec:succ-elim})\footnote{Inspired by the best-arm identification literature \cite{even2006action}.} that performs uniform allocation on only those arms that have not yet been selected by BH, and $iii$) Algorithm 1 (UCB).
Algorithm 1 and the BH selection rule for all algorithms use $\phi(t,\delta) = \sqrt{\tfrac{2\log(1/\delta)+6 \log\log(1/\delta) + 3 \log(\log(e t/2))}{t}}$ from \cite[Theorem 8]{kaufmann2016complexity}. In addition, we ran BH at level $\delta$ instead of $\delta/(6.4\log(36/\delta))$ as discussed in section \ref{sec:main_results}.
Here we present the sample complexity for TPR+FDR with $\delta=0.05$ and different parameterizations of $\mu$, $n$, $|\H_1|$.\\[6pt]
\begin{tabular}{l l l l l} 
\includegraphics[width=.252\textwidth,trim={.6cm .7cm .7cm .7cm},clip]{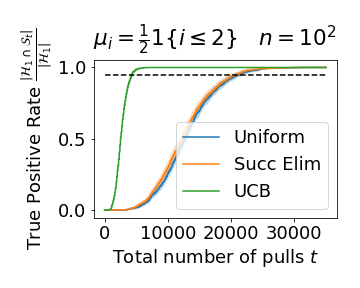} &
\includegraphics[width=.201\textwidth,trim={.3cm .7cm .7cm .4cm},clip]{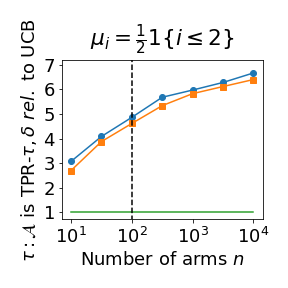} &
\includegraphics[width=.18\textwidth,trim={1.3cm .7cm .7cm .4cm},clip]{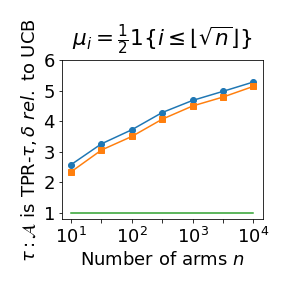} &
\includegraphics[width=.18\textwidth,trim={1.3cm .7cm .7cm .4cm},clip]{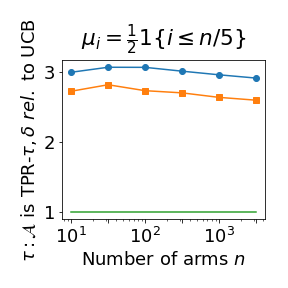} &
\includegraphics[width=.18\textwidth,trim={1.3cm .7cm .7cm .4cm},clip]{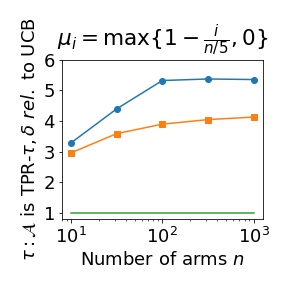}
\end{tabular}

The first  panel shows an empirical estimate of $\E[ \frac{|\calS_t \cap \H_1|}{|\H_1|} ]$ at each time $t$ for each algorithm, averaged over 1000 trials. The black dashed line on the first panel denotes the level $\E[ \frac{|\calS_t \cap \H_1|}{|\H_1|} ]=1-\delta = .95$, and corresponds to the dashed black line on the second panel. The right four panels show the number of samples each algorithm takes before the true positive rate exceeds $1-\delta=.95$, relative to the number of samples taken by UCB, for various parameterizations.
Panels two, three, and four have $\Delta_i=\Delta$ for $i \in \H_1$ while panel five is a case where the $\Delta_i$'s are linear for $i \in \H_1$.
While the differences are most clear on the second panel when $|\H_1| =2= o(n)$, over all cases UCB uses at least $\approx3$ times fewer samples than uniform and SE.
For FDR+TPR, Appendix~\ref{sec:succ-elim} shows uniform sampling roughly has a sample complexity that scales like $n \Delta^{-2} \log(\tfrac{n}{|\H_1|})$ while SE's is upper bounded by $\min\{ n \Delta^{-2} \log(\tfrac{n}{|\H_1|}), (n-|\H_1|)\Delta^{-2} \log(\tfrac{n}{|\H_1|}) + \sum_{i \in \H_1} \Delta_i^{-2} \log(n )\}$. Comparing with Theorem~\ref{thm:FDR_TPR} for the difference cases (i.e., $|\H_1| =2, \sqrt{n}, n/5$) provides insight into the relative difference between UCB, uniform, and SE on the different panels.

\subsection*{Acknowledgments}
This work was informed and inspired by early discussions with Aaditya Ramdas on methods for controlling the false discovery rate (FDR) in multiple testing; we are grateful to have learned from a leader in the field. We also thank him for his careful reading and feedback. We'd also like to thank Martin J. Zhang for his input.
We also thank the leading experimentation and A/B testing platform on the web, \emph{Optimizely}, for its support, insight into its customers' needs, and for committing engineering time to implementing this research into their platform \cite{optimizely}.
In particular, we thank Whelan Boyd, Jimmy Jin, Pete Koomen, Sammy Lee, Ajith Mascarenhas, Sonesh Surana, and Hao Xia at Optimizely for their efforts.



\clearpage
\bibliographystyle{unsrt}
\bibliography{bandit_FDR}

\newpage
\appendix

\section{Analysis Preliminaries}

Recall that for any $i \in [n]$ we assume the existence of a function $\phi$ such that for any $\delta$ we have
$\P\left( \bigcap_{t=1}^\infty \{ |\widehat{\mu}_{i,t} - \mu_i| \leq \phi(t,\delta) \}  \right) \geq 1-\delta$.
We assume that $\phi( t , \delta )$ is \textbf{non-increasing in its second argument} and that there exists an absolute constant $c_\phi$ such that $\phi( t , \delta ) \leq \sqrt{\frac{ c_\phi \log( \log_2(2 t)/\delta)}{t}}$.
It suffices to take $c_\phi=4$ but there are much sharper known bounds that should be used in practice, see \cite{jamieson2014lil,Balsubramani2014,kaufmann2016complexity,tanczosnowak2017}.
Moreover, define its \textbf{inverse} $\phi^{-1}(\epsilon,\delta) = \min\{ t : \phi(t,\delta) \leq \epsilon)$.
For the same constant $c_\phi$, it can be shown that $\phi^{-1}(\epsilon,\delta) \leq c_\phi \epsilon^{-2} \log ( 2 \log(\tfrac{e c_\phi  \epsilon^{-2}}{\delta})/\delta) \leq c \epsilon^{-2} \log( \log(\epsilon^{-2})/\delta)$ for a sufficiently large constant $c$ (and any $\epsilon,\delta<1/4$).


The technical challenges in this work revolve around arguments that \emph{avoid union bounds}.
By union bounding over all $n$ arms we have
\begin{align*}
\textstyle\P\left( \bigcup_{i=1}^n \bigcup_{t=1}^\infty \{ |\widehat{\mu}_{i,t} - \mu_i| \leq \phi(t,\tfrac{\delta}{n}) \} \right) \leq \sum_{i=1}^n \P\left( \bigcup_{t=1}^\infty \{ |\widehat{\mu}_{i,t} - \mu_i| \leq \phi(t,\tfrac{\delta}{n}) \}\right) \leq n \tfrac{\delta}{n} = \delta
\end{align*}
which says that with probability at least $1-\delta$, the deviations of \emph{all} arms after any number of samples $s$ are bounded by $\phi(s,\delta/n)$.
This is attractive because we can easily upper bound the number of times we need to sample an arm $i$ before its empirical mean is within $\epsilon$ of its true mean by $\phi^{-1}(\epsilon, \tfrac{\delta}{n}) \approx \epsilon^{-2} \log( \tfrac{n}{\delta} \log(\epsilon^{-2}))$.
However, we note that this number of samples scales as $\log(n)$, and this $\log(n)$ will persist in any sample complexity result in any analysis that uses such a union bounding technique.
That is, using union bounds in a naive way leads to sample complexities of at least $\Delta^{-2} n \log(n)$, which is no better than uniform sampling when considering FWER+FWPD (see Theorem \ref{thm:uniform_fwer_lower_bound})!

To avoid union bounds, we use techniques developed for best arm and top-k identification for multi-armed bandits \cite{jamieson2014lil,simchowitz2017simulator}.
Define the random variable
\begin{align}\label{eqn:rho_def}
\rho_i = \sup \left\{ \rho \in (0,1] : \bigcap_{t=1}^\infty \{ |\widehat{\mu}_{i,t}-\mu_i| \leq \phi(t,\rho)\} \right\}
\end{align}
and note that by the definition of $\phi$, we have $\P(\rho_i \leq x) \leq x$ for any $x \in (0,1]$.
That is, each $\rho_i$ is an \textbf{independent sub-uniformly distributed random variable}.
However, note that the $\rho_i$ random variables are \emph{not} $p$-values, that is, they are unrelated to $P_{i,t}$.
We will often make use of the fact that $\mu_i - \phi(t,\rho_i) \leq \widehat{\mu}_{i,t} \leq \mu_i + \phi(t,\rho_i)$ which is simply sandwiching a random quantity by two other random quantities. Furthermore, by definition $|\widehat{\mu}_{i,t} - \mu_i|\leq \phi(t, \delta)$ implies $\rho_i \geq
\delta$.
While a union bound would be equivalent to saying $\P(\bigcup_{i=1}^n \{ \rho_i \leq \tfrac{\delta}{n} \} ) \leq \delta$, we avoid union bounds by reasoning about only the statements we need.
For instance, instead of guaranteeing \emph{all} the $\rho_i$ are bounded by a single value, we will make guarantees about how \emph{most} behave, and argue that this is sufficient. Examples of such statements are given in events $\mathcal{E}_3, \mathcal{E}_{4,0}, \mathcal{E}_{4,1}$ defined below.
This strategy has been applied successfully to remove extraneous $\log$ factors (for instance, \cite{simchowitz2017simulator} improved the top-$k$ algorithm and analysis initially proposed in \cite{kalyanakrishnan2012pac}).

Our proofs may upper bound an expression by another with a leading constant $c$, which may become $c'$ on the next line, then $c''$, and so on.
This is meant to indicate that the constant is changing from line to line, but the next display may revert back to using $c$ and this is not the same $c$ as before.
All constants are independent of problem parameters.
This strategy is principally used to hide the ugliness of inverting the $\phi$ function, which arises only in the analysis and not the algorithm itself.

The proofs that follow are meant to be read sequentially, as some of the proofs will reuse the same lemmas.
All four theorems claim that FDR is controlled for the set $\S_t$ in the algorithm.
This is independent of the sampling scheme because of the defined anytime $p$-values.

\section{Anytime FDR for Anytime $p$-values}
We begin our analysis by proving a few general statements about anytime $p$-values and FDR control for the set $\S_t$.

For $p$-values $p_1,\dots,p_n$ and null set $\H_0 \subseteq [n]$ define
\begin{align*}
    R_k &= \{ i : p_i \leq \alpha \tfrac{k}{n} \}\\
    FDP(R_k) &= \frac{\sum_{i \in \H_0} 1\{p_i \leq \alpha \tfrac{k}{n} \}}{|R_k| \vee 1} \\
    \widehat{FDP}(R_k) &= \frac{n \max_{i \in R_k} p_i}{|R_k| \vee 1}
\end{align*}

The celebrated result of Benjamini-Hochberg says that under the assumption that the null p-values are independent and sub-uniformly distributed, if $\widehat{k} = \max\{ k : \widehat{FDP}(R_k) \leq \alpha \}$ then $\E[ FDP(R_{\widehat{k}}) ] \leq \alpha$.
The following theorem provides a bound on the expected false discovery proportion for \emph{any} $k$ such that $\widehat{FDP}(R_k) \leq \alpha$.

\begin{reptheorem}{thm:newFDRcontrol}
Fix $\alpha \in (0,e^{-1})$.
For $n\geq 2$ let $p_1,\dots,p_n$ be random variables such that $\{ p_i \}_{i\in \H_0}$ are independent and sub-uniformly distributed so that $\max_{i \in \H_0} \P(p_i \leq x ) \leq x$. For any $k \in \{0,1,\dots,n\}$, if $R_k = \{ i : p_i \leq \alpha \frac{k}{n} \}$ then
\begin{align*}
    \E\left[ \max_{k : \widehat{FDP}(R_k) \leq \alpha} FDP(R_k) \right]
    &\leq \frac{|\H_0|\alpha}{n} \left( 2  \log( \tfrac{2n}{|\H_0|\alpha }) + \log( 8 e^5\log( \tfrac{8n}{|\H_0| \alpha} ) ) \right) \\
    &\leq 4 \alpha \log( 9/\alpha).
\end{align*}
In other words, any procedure that chooses a set $\{ i : p_i \leq \frac{\alpha k}{n} \}$   for any $k$ satisfying $|\{ i : p_i \leq \frac{\alpha k}{n} \}| \geq k$ is FDR controlled at level $O(\alpha \log(1/\alpha))$.

\end{reptheorem}

\begin{proof}
Define
\begin{align*}
    p_\ell^0 &= \min \{t : \sum_{i\in \H_0}1\{p_i \leq t \} \geq \ell\}.
\end{align*}
so that $p_\ell^0$ is the value of the $\ell$th largest $p$-value in $\H_0$.
Note that
\begin{align*}
    \max_{k : \widehat{FDP}(R_k) \leq \alpha} FDP(R_k)  =& \max_{k : \widehat{FDP}(R_k) \leq \alpha} \frac{\sum_{i \in \H_0} 1\{p_i \leq \alpha \tfrac{k}{n} \}}{|R_k|}\\
    =&  \max_{k : \widehat{FDP}(R_k) \leq \alpha}  \sum_{\ell=1}^{|\H_0|} \frac{\ell}{|R_k|} \1\left\{ \ell = \sum_{i \in \H_0} 1\{p_i \leq \alpha \tfrac{k}{n} \} \right\} \\
    \leq& \max_{k : \widehat{FDP}(R_k) \leq \alpha}  \sum_{\ell=1}^{\ell_0} \frac{\ell}{|R_k|} \1\left\{ \ell = \sum_{i \in \H_0} 1\{p_i \leq \alpha \tfrac{k}{n} \} \right\} \\
    &+  \max_{k : \widehat{FDP}(R_k) \leq \alpha}  \sum_{\ell=\ell_0+1}^{|\H_0|} \frac{\ell}{|R_k|} \1\left\{ \ell = \sum_{i \in \H_0} 1\{p_i \leq \alpha \tfrac{k}{n} \} \right\}
\end{align*}
for some $\ell_0 \in \{1,\dots,|\H_0|\}$ to be defined later, where the last line follows from the fact that $\max_k ( f(k) + g(k) ) \leq \max_k  f(k) + \max_k  g(k)$.
If $\ell = \sum_{i \in \H_0} 1\{p_i \leq \alpha \tfrac{k}{n} \}$ then $\max_{i \in R_k} p_i$ is at least as large as the $\ell$th largest $p$-value in $\{ p_i : i \in \H_0 \}$ (i.e., $p_\ell^0$) and
\begin{align*}
    \frac{1\left\{ \ell = \sum_{i \in \H_0} 1\{p_i \leq \alpha \tfrac{k}{n} \}\right\}}{\max_{i \in R_k} p_i} \leq \frac{1\left\{ \ell = \sum_{i \in \H_0} 1\{p_i \leq \alpha \tfrac{k}{n} \}\right\}}{ p_\ell^0 }.
\end{align*}
Let $u_1,\dots,u_{|\H_0|}$ be iid uniformly distributed random variables on $[0,1]$.
Note that for every $i \in \H_0$ and $j=1,\dots,|\H_0|$ we have $\P(p_i \leq x) \leq x = \P(u_j \leq x)$ where we have used the assumption that each $p_i$ is sub-uniform randomly distributed.
Let $u_{(i)}$ be the $i$th largest $u_i$ such that $u_{(1)} \leq u_{(2)} \leq \dots \leq u_{(|\H_0|)}$.
Note that for any $\ell = 1,\dots,|\H_0|$ we have
\begin{align*}
    \P\left(p_\ell^0 \leq x \right) &= \P( \sum_{i \in \H_0} 1\{ p_i \leq x \} \geq \ell )
    \leq \P( \sum_{i=1}^{|\H_0|} \1\{ u_i \leq x \} \geq \ell) = \P\left(u_{(\ell)} \leq x \right)
\end{align*}
since each $\1\{ p_i \leq x\}$ is an independent Bernoulli random variable with parameter $\P(p_i \leq x) \leq x = \P(u_j \leq x)$ for any $j \in \{1,\dots,|\H_0|\}$.

We then observe that
\begin{align*}
    \P\left(u_{(\ell)} \leq x \right) & \leq \P( \sum_{i=1}^{|\H_0|} \1\{ u_i \leq x \} \geq \ell)\\
    &\leq \exp\left(- |\H_0| KL( \tfrac{\ell}{|\H_0|} , x) \right) \\
    &\leq \exp\left(- |\H_0| \frac{\left(\tfrac{\ell}{|\H_0|} - x\right)^2}{2 \ell /|\H_0|} \right) \\
    &= \exp\left(- \frac{\ell}{2} \left(1 - \frac{|\H_0|x}{\ell}\right)^2 \right)
\end{align*}
where the second inequality is Chernoff's bound, and the next line follows from $KL(x,y) \geq \frac{(x-y)^2}{2x}$ for $x > y$.
For each $\ell=1,2,\dots,|\H_0|$ pick $x = \frac{\ell }{2|\H_0|}$ so that $\P\left(p_\ell^0 \leq \frac{\ell }{2|\H_0|} \right) \leq e^{-\ell/8}$.
Also, if we define the event
\begin{align*}
    \mathcal{E}_\ell := \left\{  p_\ell^0 \geq \frac{\ell }{2|\H_0|} \right\}
\end{align*}
then $\P(\mathcal{E}_\ell) \geq 1-e^{-\ell/8}$.
If $\mathcal{E} = \cap_{\ell=\ell_0+1}^{|\H_0|} \mathcal{E}_\ell$ then
\begin{align*}
    \hspace{.5in}&\hspace{-.5in}\max_{k : \widehat{FDP}(R_k) \leq \alpha}  \sum_{\ell=\ell_0+1}^{|\H_0|} \frac{\ell \1\left\{ \ell = \sum_{i \in \H_0} 1\{p_i \leq \alpha \tfrac{k}{n} \} \right\} }{|R_k|}\\
    &\leq \1\{ \mathcal{E}^c \} + \1\{ \mathcal{E} \} \max_{k : \widehat{FDP}(R_k) \leq \alpha}  \sum_{\ell=\ell_0+1}^{|\H_0|} \frac{\ell \1\left\{ \ell = \sum_{i \in \H_0} 1\{p_i \leq \alpha \tfrac{k}{n} \} \right\} }{|R_k|}   \\
    &\leq \1\{ \mathcal{E}^c \} + \1\{ \mathcal{E} \} \max_{k : \widehat{FDP}(R_k) \leq \alpha}  \sum_{\ell=\ell_0+1}^{|\H_0|} \frac{\ell \1\left\{ \ell = \sum_{i \in \H_0} 1\{p_i \leq \alpha \tfrac{k}{n} \} \right\} }{n \max_{i \in R_k} p_i} \alpha  \\
    &\leq \1\{ \mathcal{E}^c \} + \1\{ \mathcal{E} \} \max_{k : \widehat{FDP}(R_k) \leq \alpha}  \sum_{\ell=\ell_0+1}^{|\H_0|} \frac{\ell \1\left\{ \ell = \sum_{i \in \H_0} 1\{p_i \leq \alpha \tfrac{k}{n} \} \right\} }{n p_\ell^0} \alpha  \\
    &\leq \1\{ \mathcal{E}^c \} + \1\{ \mathcal{E} \} \max_{k }  \sum_{\ell=\ell_0+1}^{|\H_0|} \frac{2 |\H_0| \1\left\{ \ell = \sum_{i \in \H_0} 1\{p_i \leq \alpha \tfrac{k}{n} \} \right\}}{n } \alpha   \\
    &\leq \1\{ \mathcal{E}^c \} + \1\{ \mathcal{E} \} \frac{2 |\H_0|}{n} \alpha \, \max_{k }  \sum_{\ell=\ell_0+1}^{|\H_0|}  \1\left\{ \ell = \sum_{i \in \H_0} 1\{p_i \leq \alpha \tfrac{k}{n} \} \right\}  \\
    &\leq \1\{ \mathcal{E}^c \} + \frac{2 |\H_0|}{n}\alpha.
\end{align*}
Note that the only randomness on the right-hand-side is $ \1\{ \mathcal{E}^c \}$.
After taking expectations on both sides we observe that
\begin{align*}
\E[\1\{ \mathcal{E}^c \}] = \E\left[\1\{ \cup_{\ell=\ell_0+1}^{|\H_0|} \mathcal{E}_\ell^c \}\right] = \sum_{\ell=\ell_0+1}^{|\H_0|} \P(\mathcal{E}_\ell^c) \leq \sum_{\ell=\ell_0+1}^{|\H_0|} e^{-\ell/8} \leq 8 e^{-\ell_0/8}.
\end{align*}
Thus,
\begin{align*}
    \E\left[ \max_{k : \widehat{FDP}(R_k) \leq \alpha}  \sum_{\ell=\ell_0+1}^{|\H_0|} \frac{\ell \1\left\{ \ell = \sum_{i \in \H_0} 1\{p_i \leq \alpha \tfrac{k}{n} \} \right\} }{|R_k|} \right] \leq 8 e^{-\ell_0/8} + \frac{2 |\H_0|}{n}\alpha.
\end{align*}


On the other hand, for any $k$ with $ \widehat{FDP}(R_k) = \frac{n \max_{i \in R_k} p_i}{|R_k| \vee 1} \leq \alpha$ and $\ell = \sum_{i \in \H_0} 1\{p_i \leq \alpha \tfrac{k}{n} \}$ we have that the $\ell$th $p$-value is in $R_k$ and
$\frac{n p_\ell^0}{\alpha} \leq \frac{n \max_{i \in R_k} p_i}{\alpha} \leq |R_k|$ where $|R_k|\geq 1$ necessarily.
Consequently, $|R_k| \geq \lceil n p_\ell^0 /\alpha \rceil$ so that
\begin{align*}
    \max_{k : \widehat{FDP}(R_k) \leq \alpha} \sum_{\ell=1}^{\ell_0} \frac{\ell}{|R_k|} \1\left\{ \ell = \sum_{i \in \H_0} 1\{p_i \leq \alpha \tfrac{k}{n} \} \right\} &\leq \max_{k : \widehat{FDP}(R_k) \leq \alpha} \sum_{\ell=1}^{\ell_0} \frac{\ell}{|R_k|} \1\left\{ \ell = \sum_{i \in \H_0} 1\{p_i \leq \alpha \tfrac{k}{n} \} \right\} \\
    &\leq  \sum_{\ell=1}^{\ell_0} \max_{k}   \frac{\ell}{|R_k|} \1\left\{ \ell = \sum_{i \in \H_0} 1\{p_i \leq \alpha \tfrac{k}{n} \} \right\}  \\
    &\leq \sum_{\ell=1}^{\ell_0} \max_{k}   \frac{\ell}{\lceil n p_\ell^0 /\alpha \rceil} \1\left\{ \ell = \sum_{i \in \H_0} 1\{p_i \leq \max_{i \in R_k} p_i \} \right\}  \\
    &= \sum_{\ell=1}^{\ell_0} \frac{\ell}{\lceil n p_\ell^0 /\alpha \rceil}
\end{align*}
As above, we use the fact that each $p_\ell^0$ is stochastically dominated by $u_{(\ell)}$ so that
\begin{align*}
    \E\left[ \frac{\ell}{\lceil n p_\ell^0 /\alpha \rceil} \right] &=  \int_{x=0}^\infty \P\left( \frac{\ell}{\lceil n p_\ell^0 /\alpha \rceil} \geq x \right) dx\\
    &=  \int_{x=0}^\infty \P\left( \lceil n p_\ell^0 /\alpha \rceil \leq \ell/x \right) dx \\
    &=  \int_{x=0}^\infty \P\left(  n p_\ell^0 /\alpha  \leq \lfloor \ell/x \rfloor \right) dx \\
    &\leq  \int_{x=0}^\infty \P\left(  n u_{(\ell)} /\alpha  \leq \lfloor \ell/x \rfloor \right) dx\\
    &=  \int_{x=0}^\infty \P\left( \lceil n u_{(\ell)} /\alpha \rceil \leq \ell/x \right) dx \\
    &= \E\left[ \frac{\ell}{\lceil n u_{(\ell)} /\alpha \rceil} \right]
\end{align*}
since $ \lceil a \rceil > b \iff a > \lfloor b \rfloor$.
Thus,
\begin{align*}
\E\left[ \max_{k : \widehat{FDP}(R_k) \leq \alpha} \sum_{\ell=1}^{\ell_0} \frac{\ell}{|R_k|} \1\left\{ \ell = \sum_{i \in \H_0} 1\{p_i \leq \alpha \tfrac{k}{n} \} \right\} \right] \leq \sum_{\ell=1}^{\ell_0}\E\left[  \frac{\ell}{ \lceil n p_\ell^0 /\alpha \rceil} \right] \leq \sum_{\ell=1}^{\ell_0}\E\left[  \frac{\ell}{\lceil n u_{(\ell)} /\alpha \rceil} \right].
\end{align*}

For notational ease, let $m=|\H_0|$.
Recall that the PDF of the $\ell$-th order statistic of $m$ iid uniform random variables on $[0,1]$ is given by,
\[\frac{d\P(u_{(\ell)} \leq x)}{dx} = \ell\binom{m}{\ell}(1-x)^{m-\ell}x^{\ell-1}.\]
The following simple bound will be useful for small $\ell$ and $y$:
\begin{align*}
    \P(u_{(\ell)} \leq y) = \int_{x=0}^y \ell\binom{m}{\ell}(1-x)^{m-\ell}x^{\ell-1} dx \leq \binom{m}{\ell} y^\ell \leq (m y)^{\ell}.
\end{align*}

 First we consider the case when $\ell>1$.
\begin{align*}
    \E\left[ \frac{\ell}{\lceil n u_{(\ell)} /\alpha \rceil} \right] &= \sum_{k=1}^{\infty} \frac{1}{k} \P\left(\frac{(k-1)\alpha}{n}\leq u_{(\ell)}\leq \frac{k\alpha}{n}\right) \\
    &\leq \P\left({u_{(\ell)}}\leq \frac{\alpha}{n}\right)+\int_{x=\frac{\alpha}{n}}^1 \frac{\alpha}{nx} d\P(u_{(\ell)} \leq x)\\
    &=\P\left(u_{(\ell)}\leq\frac{\alpha}{n}\right)+\frac{\alpha}{n}\ell\binom{m}{\ell}\int_{x=\frac{\alpha}{n}}^1 (1-x)^{m-\ell}x^{\ell-2}dx\\
    &\leq \left( \frac{\alpha m}{n} \right)^\ell +\frac{\alpha}{n}\frac{\ell\binom{m}{\ell}}{(\ell-1)\binom{m-1}{\ell-1}}\int_{\alpha/n}^1 (\ell-1)\binom{m-1}{\ell-1}(1-x)^{m-\ell}x^{\ell-2}dx\\
    &\leq \left( \frac{\alpha m}{n} \right)^\ell +\frac{\alpha}{n}\frac{m!}{(\ell-1)!(m-\ell)!}\frac{(\ell-2)!(m-\ell)!}{(m-1)!}\\
    &= \left( \frac{\alpha m}{n} \right)^\ell +\frac{\alpha}{n}\frac{m}{\ell-1}
\end{align*}
where going from the third line to the fourth, the integrand is just the PDF of a Beta-distribution so the integral can be bounded away by 1. When $\ell=1$, the argument is slightly different. Firstly note that the PDF is a decreasing function, hence
\begin{align*}
    \P\left(\frac{(k-1)\alpha}{n}\leq u_{(1)}\leq \frac{k\alpha}{n}\right)
    &\leq \frac{\alpha}{n}m \left(1-\frac{\alpha(k-1)}{n}\right)^{m-1}\\
    &\leq \frac{\alpha}{n} me^{-(m-1)\frac{\alpha(k-1)}{n}}
\end{align*}
using the fact that $e^{-x}>1-x$. Hence,
\begin{align*}
    \E\left[ \frac{1}{\lceil n u_{(1)} /\alpha \rceil} \right] &= \sum_{k=1}^{\lceil n /\alpha \rceil} \frac{1}{k} \P\left(\frac{(k-1)\alpha}{n}\leq u_{(1)}\leq \frac{k\alpha}{n}\right)\\
    &\leq \frac{m\alpha}{n}\sum_{k=1}^{\infty}   \frac{1}{k} e^{-(m-1)\frac{\alpha(k-1)}{n}}\\
    &= \frac{m\alpha}{n}( 1 + \sum_{k=1}^{\infty}   \frac{1}{k+1} e^{-\frac{(m-1) \alpha }{n} k})\\
    &= \frac{m\alpha}{n} e^{\frac{(m-1) \alpha }{n}} \log\left( \frac{1}{1 -  e^{-\frac{(m-1) \alpha }{n}} }\right)\\
    &\leq 2 \frac{m\alpha}{n}\log( \tfrac{n}{(m-1) \alpha })
\end{align*}
where the last line holds for $\alpha (m-1) /n \leq e^{-1}$ by observing that $e^x \log(1-e^{-x})/\log(x)$ is increasing and less than $2$ at $x=e^{-1}$.
On the other hand, if $m=1$ then in the above display $\sum_{k=1}^{\lceil n  /\alpha \rceil} 1/k \leq \log( \lceil n  /\alpha \rceil e) \leq 2\log(n/\alpha)$ for $\alpha \leq e^{-1}$ and $n \geq 2$.
Thus,
\begin{align*}
    \E\left[ \frac{1}{\lceil n u_{(1)} /\alpha \rceil} \right] &= \sum_{k=1}^{\lceil n /\alpha \rceil} \frac{1}{k} \P\left(\frac{(k-1)\alpha}{n}\leq u_{(1)}\leq \frac{k\alpha}{n}\right) \\
    &\leq \begin{cases} 2 \frac{m\alpha}{n}\log( \tfrac{n}{(m-1) \alpha }) & \text{ if } m > 1 \\ 2\frac{\alpha}{n}\log(\tfrac{n}{\alpha}) & \text { if } m=1 \end{cases}\\
    &\leq  2  \frac{m\alpha}{n}\log( \tfrac{2n}{m\alpha })
\end{align*}
Plugging it all in we get
\begin{align*}
    \sum_{\ell=1}^{\ell_0} \E\left[ \frac{\ell}{\lceil n u_{(\ell)} /\alpha \rceil} \right] &\leq 2  \frac{m\alpha}{n}\log( \tfrac{2n}{m\alpha })  + \sum_{\ell=2}^{\ell_0} \left( \left( \frac{\alpha m}{n} \right)^\ell +\frac{\alpha}{n}\frac{m}{\ell-1} \right) \\
    &\leq 2  \frac{m\alpha}{n}\log( \tfrac{2n}{m\alpha }) + \left( \frac{\alpha m}{n} \right)^2/(1- \frac{\alpha m}{n}) + \frac{\alpha m}{n} \log(\ell_0 e)
\end{align*}

Putting it all together, we conclude that
\begin{align*}
    \E\left[ \max_{k : \widehat{FDP}(R_k) \leq \alpha} FDP(R_k) \right] &\leq 8 e^{-\ell_0/8} + \frac{2 |\H_0|}{n}\alpha +  \sum_{\ell=1}^{\ell_0} \E\left[ \frac{\ell}{\lceil n u_{(\ell)} /\alpha \rceil} \right] \\
    &\leq 8 e^{-\ell_0/8} + \frac{|\H_0|}{n}\alpha \left( 2 \log( \tfrac{2n}{|\H_0|\alpha }) + \log(\ell_0 e^4) \right) \\
    &\leq \frac{|\H_0|\alpha}{n} \left( 2  \log( \tfrac{2n}{|\H_0|\alpha }) + \log( 8 e^5\log( \tfrac{8n}{|\H_0| \alpha} ) ) \right) \\
    &= O\left( \tfrac{|\H_0| \alpha}{n} \log( \tfrac{n}{|\H_0|\alpha }) \right)
\end{align*}
Where the last lines have taken $\ell_0 = 8 \log( \tfrac{8n}{|\H_0| \alpha} )$.
\end{proof}

In this section, consider a hypothesis test between $\H_0$ and $\H_1$. Let $\{P_{i,t}\}_{t=1}^\infty, P_{i,t}\in (0,1]$ be a collection of random variables such that for $i\in \mathcal{H}_0$, $\{P_{i,t}\}_{i=1}^\infty$ are sub-uniformly distributed anytime $p$-values, i.e. $\P(\cup_{t=1}^{\infty}\{P_{i,t}\leq x \}) \leq x$.
Let $\S_t$ be the set of discoveries following the Benjamini-Hochberg procedure at confidence level $\delta$ at time $t$ on the $p$-values $P_{i,T_i(t)}$, $1\leq i\leq n$. The following lemma employees the previous Theorem to guarantees FDR-control.
\begin{lemma}\label{lem:expected_FDR}
 For all times $t\geq 1$,  FDR is controlled at level $\delta$ so that $\E[\frac{|\calS_t \cap \H_0|}{|\calS_t| \vee 1}] \leq \delta$.
\end{lemma}
\begin{proof}

For any $i \in \H_0$ so that $\mu_i \leq \mu_0$ define $P_{i,*} = \inf_{t\geq 1} P_{i,t}$.
Then for any $i \in \H_0$ and $\delta \in (0,1)$
\begin{align*}
\P(P_{i,*}\leq \delta) &= \P\left( \bigcup_{t=1}^\infty \{ P_{i,t} \leq \delta \} \right) = \P\left( \bigcup_{t=1}^\infty \{  \widehat{\mu}_{i,t} -  \phi(t,\delta) \geq \mu_0 \} \right) \\
&\leq \P\left( \bigcup_{t=1}^\infty \{  \widehat{\mu}_{i,t} -  \phi(t,\delta) \geq \mu_i \} \right)
\leq \P\left( \bigcup_{t=1}^\infty \{  |\widehat{\mu}_{i,t} - \mu_i| \geq \phi(t,\delta) \right)
\leq \delta.
\end{align*}
We observe that $\{ P_{i,*} \}_{i \in \H_0}$ are sub-uniformly distributed random variables but are also \emph{independent} since they only depend on the rewards of arm $i \in \H_0$.
Thus, we may apply the above proposition at time $t$ with $\{ P_{i,*} : i \in \H_0 \} \cup \{ P_{i,T_i(t)} : i \in \H_1 \}$ to conclude that FDR would be controlled for any of the prescribed values of $k$ (including the largest, which would be equivalent to BH).
We need to show that the BH procedure controls FDR at time $t$ with $\{ P_{i,T_i(t)} : i \in \H_0 \} \cup \{ P_{i,T_i(t)} : i \in \H_1 \}$.

As described in Section \ref{sec:false_alarm_control}, the BH procedure selects the $\widehat{k}$ smallest $p$-values where $$\widehat{k} = \max\{ k : |\{i:P_{i,T_i(t)} \leq \delta\tfrac{k}{n}\}| \geq k \}.$$
By definition
\begin{align*}
\widehat{k} &= |\{i \in \H_0 :P_{i,T_i(t)} \leq \delta\tfrac{\widehat{k}}{n}\} \cup \{ i \in \H_1 : P_{i,T_i(t)} \leq \delta\tfrac{\widehat{k}}{n} \}| \\
&\leq |\{ i \in \H_0 : P_{i,*} \leq \delta\tfrac{\widehat{k}}{n} \} \cup \{ i \in \H_1 : P_{i,T_i(t)} \leq \delta\tfrac{\widehat{k}}{n} \}|
\end{align*}
since $P_{i,*} \leq P_{i,T_i(t)}$ for all $i \in \H_0$.
Thus, since the $\{ P_{i,*} \}_{i \in \H_0}$ are independent and sub-uniformly distributed and \emph{at least} $\widehat{k}$ of $\{ i \in \H_0 : P_{i,*} \leq \delta\tfrac{\widehat{k}}{n} \} \cup \{ i \in \H_1 : P_{i,T_i(t)} \leq \delta\tfrac{\widehat{k}}{n} \}$
are below the threshold $\frac{\widehat{k}\delta}{n}$ we apply Theorem\ref{thm:newFDRcontrol} to conclude that the FDR of $\{ i \in \H_0 : P_{i,*} \leq \delta\tfrac{\widehat{k}}{n} \} \cup \{ i \in \H_1 : P_{i,T_i(t)} \leq \delta\tfrac{\widehat{k}}{n} \}$ is bounded by $\delta$.

We observe that
\begin{align*}
FDR( \{ i \in [n] : P_{i,T_i(t)} \leq \delta\tfrac{\widehat{k}}{n} \}) &= \E\left[ \frac{\sum_{i\in \H_0} \1\{P_{i,T_i(t)} \leq \delta\frac{\widehat{k}}{n}\}}{\sum_{i\in \H_0} \1\{P_{i,T_i(t)} \leq \delta\frac{\widehat{k}}{n}\} + \sum_{i\in \H_1} \1\{P_{i,T_i(t)} \leq \delta\frac{\widehat{k}}{n}\}} \right] \\
&\leq \E\left[ \frac{\sum_{i\in \H_0} \1\{P_{i,\ast} \leq \delta\frac{\widehat{k}}{n}\}}{\sum_{i\in \H_0} \1\{P_{i,\ast} \leq \delta\frac{\widehat{k}}{n}\} + \sum_{i\in \H_1} \1\{P_{i,T_i(t)} \leq \delta\frac{\widehat{k}}{n}\}} \right] \\
&= FDR( \{ i \in \H_0 : P_{i,*} \leq \delta\tfrac{\widehat{k}}{n} \} \cup \{ i \in \H_1 : P_{i,T_i(t)} \leq \delta\tfrac{\widehat{k}}{n} \} ) \\
&\leq \delta
\end{align*}
which is precisely what we wished to prove, where the first inequality follows from  $\frac{a}{a+c} \leq \frac{a + b}{a+b+c}$ for positive numbers $a,b,c$, and the fact that $\sum_{i\in \H_0} \1\{P_{i,T_i(t)} \leq \delta\frac{\widehat{k}}{n}\} \leq \sum_{i\in \H_0} \1\{P_{i,\ast} \leq \delta\frac{\widehat{k}}{n}\}$ due to $P_{i,*} \leq P_{i,T_i(t)}$ for all $i \in \H_0$.

\end{proof}

The following lemma is a stronger result, providing an anytime high-probability bound on the false discovery proportion and the size of the discovered set.
The proof follows from considering not the $p$-values, $P_{i,t}$ at any given time, but rather the random variable which is the worst-case $p$-values over all time.

\begin{lemma}\label{lem:fdr_high_prob}  Let $T_i:\mathbb{N}\rightarrow \mathbb{N}$ be an arbitrary function for any $i$. Recall $\delta' = \delta/(6.4\log(36/\delta))$.
 Then with probability greater than $ \geq 1-\delta'$,
\begin{align*}
\bigcap_{t=1}^\infty \left\{ |\calS_t| \leq \tfrac{1}{1-2\delta'(1+4\delta')} |\H_1| + \tfrac{4(1+4\delta')/3}{1-2\delta'(1+4\delta')} \log(\tfrac{5\log_2(n/\delta')}{\delta'}) \right\}
\end{align*}
and
\begin{align*}
\bigcap_{t=1}^\infty \left\{ \frac{|\calS_t \cap \H_0|}{|\calS_t| \vee 1} \leq \delta' \tfrac{|\H_0|}{n} +  (1+4\delta') \sqrt{ \frac{4 \delta' \tfrac{|\H_0|}{n} \log(\tfrac{\log_2(n/\delta)}{\delta'})}{ |\calS_t| \vee 1} } + \frac{ (1+4\delta')\log(\tfrac{\log_2(n/\delta')}{\delta'})}{3 (|\calS_t| \vee 1)} \right\}
\end{align*}
in particular, these events hold on $\mathcal{E}_3$ (defined in the proof), which holds with probability greater than $1-\delta'$.
\end{lemma}
\begin{proof}
As in the proof of Lemma \ref{lem:expected_FDR}, let $P_{i,\ast} = \inf_{t\geq 1} P_{i,t}$ for all $i\in \H_0$.

Define the event,
\begin{align*}
\mathcal{E}_3 &:= \bigg\{ \forall s \in (0,1]: \sum_{i \in \H_0} \1\{P_{i,\ast}\leq s\} \\
    &\hspace{.5in}\leq  s |\H_0| + (1+4s) \sqrt{2 \max\{2s, 2\delta'/n\} |\H_0| \log(\tfrac{\log_2(n/\delta')}{\delta'}) } + \tfrac{1+4s}{3} \log(\tfrac{\log_2(n/\delta')}{\delta'}) \bigg\}
\end{align*}

We have that $\P(\mathcal{E}_3) \geq 1-\delta'$ by applying Lemma~\ref{lem:brownian_bridge} found in the appendix with $c=2\delta'/n$ and $X_i = P_{i,\ast}$ for $i \in \H_0$ so that $m=|\H_0|$.
Note that there exists some threshold $\tau_t \in [0,1]$ such that $BH$ selects all indices with $P_{i,T_i(t)} \leq \tau_t$ so that $\calS_t = \S(\tau_t) := \{i \in [n]:  P_{i, T_i(t)} \leq \tau_t\}$. Then 
\begin{align*}
|\calS_t \cap \H_0| = \sum_{i \in \H_0} \1\{ P_{i,T_i(t)} \leq \tau_t\} \leq \sum_{i \in \H_0} \1\{ P_{i, \ast} \leq \tau_t\}.
\end{align*}
By definition, $\tau_t = \sup \{ s \leq 1: \frac{|\calS(s)|}{n} \delta' \geq s \}$, otherwise we take it to be $0$, so that $\tau_t \leq \delta' |\calS(\tau_t)|/ n$. We apply this inequality and $\mathcal{E}_3$ to observe
\begin{align*}
|\calS_t &\cap \H_0| \leq \tau_t |\H_0| + (1+4\tau_t) \sqrt{2 \max\{2\delta'/n, 2\tau_t\} |\H_0| \log(\tfrac{\log_2(n/\delta')}{\delta'}) } + \tfrac{1+4\tau_t}{3} \log(\tfrac{\log_2(n/\delta')}{\delta'}) \\
&\leq \delta' |\calS(\tau_t)| \tfrac{|\H_0|}{n} + (1+4\delta') \sqrt{2 \max\{2\delta'/n, 2\delta' |\calS(\tau_t)|\tfrac{1}{n}\} |\H_0| \log(\tfrac{\log_2(n/\delta')}{\delta'}) } + \tfrac{1+4\delta'}{3} \log(\tfrac{\log_2(n/\delta')}{\delta'}) \\
&= \delta' |\calS(\tau_t)| \tfrac{|\H_0|}{n} + 2 (1+4\delta') \sqrt{ \delta |\calS(\tau_t)| \tfrac{|\H_0|}{n} \log(\tfrac{\log_2(n/\delta')}{\delta'}) } + \tfrac{1+4\delta'}{3} \log(\tfrac{\log_2(n/\delta')}{\delta'}) \\
&\leq 2\delta'(1+4\delta') |\calS(\tau_t)| \tfrac{|\H_0|}{n} + \tfrac{4(1+4\delta')}{3} \log(\tfrac{\log_2(n)}{\delta'})
\end{align*}
where the last inequality follows from $a + 2\sqrt{ab} + b = (\sqrt{a} + \sqrt{b})^2 \leq 2(a+b)$.
After rearranging, because $|\calS_t| = |\calS_t \cap \H_1| + |\calS_t \cap \H_0| \leq |\H_1|  + |\calS_t \cap \H_0|$ we obtain the theorem.
\end{proof}

We will refer to event $\mathcal{E}_3$ defined in the proof above in what follows.

\section{Proof of Theorem~\ref{thm:FDR_TPR}}\label{sec:FDR_TPR_proof}
We will need the following event. Let $a \in \R_+^n$ be a fixed vector to be defined later in the proof and define
\begin{align*}
\mathcal{E}_{4, j} &:= \left\{ \sum_{i \in \H_j} a_i \log(1/\rho_i) \leq 5 \log(1/\delta) \sum_{i \in \H_j} a_i \right\} \quad{j\in \{0,1\}}\\
\end{align*}

\begin{lemma}
$\min\{\P(\mathcal{E}_{4,0}), \P(\mathcal{E}_{4,1})\} \geq 1-\delta$.
\end{lemma}
\begin{proof}
The proof is the same for $\H_0$ and $\H_1$ so we prove it for $i=1,\dots,n$.
Because for $i=1,\dots,n$ the $\rho_i$ are independent, sub-uniformly distributed random variables, we have that $Z_i = a_i \log(1/\rho_i)$ are independent random variables satisfying $\P( Z_i \geq t ) \leq \exp(-t/a_i)$. To see this, $\P(\rho_i \leq x ) \leq x$ by definition, and thus $\P( a_i\log(1/\rho_i) \geq a_i \log(1/x) ) \leq x$. Set $t=a_i \log(1/x)$ and solve for $x$.
The result follows by applying Lemma~\ref{lem:exponential_tails}.
\end{proof}

We restate the theorem using the above events.
\begin{reptheorem}{thm:FDR_TPR}[FDR, TPR]
Let $\H_1 = \{ i \in [n]: \mu_i > \mu_0\}$, $\H_0 = \{ i \in [n]: \mu_i \leq \mu_0 \}$. Define $\Delta_i = \mu_i - \mu_0$ for $i \in \H_1$, $\Delta_i=\min_{j \in \H_1} \mu_j - \mu_i$ for $i\in\H_0$, and $\Delta = \min_{i \in \H_1} \Delta_i$.
For all $t$ we have $\E[\frac{|\calS_t \cap \H_0|}{|\calS_t|}] \leq \delta$. Moreover, on $\mathcal{E}_{4,0} \cap \mathcal{E}_{4,1}$ (which holds with probability at least $1-2\delta$), there exists a $T$ such that
\begin{align*}
T &\lesssim \sum_{i \in \H_0} \Delta_i^{-2} \log( \log(\Delta_i^{-2})/\delta) + \sum_{i \in \H_1} \Delta_i^{-2}  \log( n \log(\Delta_i^{-2})/\delta)
\end{align*}
and $\E[\frac{|\calS_t \cap \H_1|}{|\H_1|}] \geq 1-\delta$ for all $t \geq T$.
Also, on the same events there exists a $T$ such that
\begin{align*}
T &\lesssim n \Delta^{-2} \log( \log(\Delta^{-2})/\delta)
\end{align*}
and $\E[\frac{|\calS_t \cap \H_1|}{|\H_1|}] \geq 1-\delta$ for all $t \geq T$.
Note, neither follows from the other.
\end{reptheorem}
\begin{proof}
Define the random set $\calI = \{ i \in \H_1 : \rho_i \geq \delta \}$.
Note, this is equivalent to $\calI=\{ i \in \H_1 : \widehat{\mu}_{i,T_i(t)} + \phi(T_i(t),\delta) \geq \mu_i \ \ \forall t \in \mathbb{N} \}$ since if $\rho_i \geq \delta$ then $\phi(T_i(t),\delta) > \phi(T_i(t),\rho_i)$ and
\begin{align*}
\widehat{\mu}_{i,T_i(t)} + \phi(T_i(t),\delta) \geq \mu_i + \phi(T_i(t),\delta) - \phi(T_i(t),\rho_i) \geq \mu_i.
\end{align*}
Our aim is to show that this ``well-behaved'' set of indices $\calI$ will be added to $\S_t$ in the claimed amount of time.
This is sufficient for the TPR result because $\E|\calI| = \sum_{i \in \H_1} \E[\1\{ \rho_i > \delta\}] = \sum_{i \in \H_1} \P( \rho_i > \delta) \geq (1-\delta) |\H_1|$ since each $\rho_i$ is a sub-uniformly distributed random variable.
To be clear, we are not claiming which \emph{particular} indices of $\H_1$ will be added to $\S_t$ (indeed, $\mathcal{I}$ is a random set), just that their number exceeds $(1-\delta)|\H_1|$ in expectation.

Note that $\calS_t \subseteq \calS_{t+1}$ for all $t$ so define $T = \min \{ t \in \mathbb{N} : \calI \subseteq \calS_{t+1} \}$ if such inclusion ever exists, otherwise let $T = \infty$.
Then
\begin{align*}
T &= \sum_{t=1}^\infty \1\{  \calI \not\subseteq \calS_t\}\\
&= \sum_{t=1}^\infty \1\{ I_t \in \H_0, \calI \not\subseteq \calS_t\} + \1\{ I_t \in \H_1\}
\end{align*}
We will bound each sum separately, starting with the first.

For any $j \in \calI$ we have
\begin{align*}
\widehat{\mu}_{j,T_j(t)} + \phi(T_j(t),\delta) \geq \mu_j + \phi(T_j(t),\delta) - \phi(T_j(t),\rho_j) \geq \mu_j.
\end{align*}
Thus $\{\calI \not\subseteq \calS_t\}$ implies that
\begin{align*}
\arg\max_{j \in \S_t^c} \widehat{\mu}_{j,T_j(t)} + \phi(T_j(t),\delta) \geq \min_{j \in \calI} \mu_j \geq \min_{j \in \H_1} \mu_j.
\end{align*}
On the other hand, for any $i \in \H_0$ we have
\begin{align*}
\widehat{\mu}_{i,T_i(t)} + \phi(T_i(t),\delta) &\leq \mu_i + \phi(T_i(t),\delta) + \phi(T_i(t),\rho_i) \\
&\leq \mu_i + 2 \phi(T_i(t),\delta \rho_i)
\end{align*}
so that $\widehat{\mu}_{i,T_i(t)} + \phi(T_i(t),\delta) \leq \min_{j \in \H_1} \mu_j = \mu_i + \Delta_i$ whenever $T_i(t) \geq \phi^{-1}(\tfrac{\Delta_i}{2}, \delta \rho_i)$.
If $T_i(t)$ were this large, the arm $i \in \H_0$ could not be pulled because its upper confidence bound would be below the upper confidence bound of an arm $j \in \calI \cap \calS_t^c$.
Thus,
\begin{align*}
\sum_{t=1}^\infty \1\{ I_t \in \H_0, \calI \not\subseteq \calS_t \} &\leq \sum_{i \in \H_0} \phi^{-1}(\tfrac{\Delta_i}{2}, \delta \rho_i) \\
&\leq \sum_{i \in \H_0} c \Delta_i^{-2} \log ( \log(\Delta_i^{-2}) / (\delta\rho_i))  \\
&\leq \sum_{i \in \H_0} c \Delta_i^{-2} \log ( \log(\Delta_i^{-2}) / \delta) + c \Delta_i^{-2} \log(1/\rho_i) \\
&\overset{\mathcal{E}_{4,0}}{\leq} \sum_{i \in \H_0} c \Delta_i^{-2} \log ( \log(\Delta_i^{-2}) / \delta) + 5 c \Delta_i^{-2} \log(1/\delta) \\
&\leq \sum_{i \in \H_0} c' \Delta_i^{-2} \log ( \log(\Delta_i^{-2}) / \delta)  \\
&\leq c'' |\H_0| \Delta^{-2} \log ( \log(\Delta^{-2}) / \delta)
\end{align*}
which concludes the upper bound on the first term.

For the second term, consider a time that $I_t \in \H_1$. Recall $\delta' = \delta/(9.6\log(3000/\delta))$.
For any $j \in \H_1$ and arbitrary $k \leq n$ by the BH procedure in the algorithm we have
\begin{align*}
\widehat{\mu}_{j,T_j(t)} - \phi(T_j(t),\delta'\tfrac{k}{n}) &\geq \mu_j - \phi(T_j(t),\delta'\tfrac{k}{n}) - \phi(T_j(t),\rho_j) \\
&\geq \mu_j - 2\phi(T_j(t),\delta' \rho_j \tfrac{k}{n})
\end{align*}
so that $\widehat{\mu}_{j,T_j(t)} - \phi(T_j(t),\delta'\rho_j \tfrac{k}{n}) \geq \mu_0$ whenever $T_j(t) \geq \phi^{-1}(\tfrac{\mu_j-\mu_0}{2},\delta'\rho_j\tfrac{k}{n})$, guaranteeing its spot in $s(k)$.
In the worst case, the arms are added one at a time to $s(k)$, instead as a group.
Thus, if $\pi$ is any map $\H_1 \rightarrow \{1,\dots,|\H_1|\}$ then
\begin{align*}
\sum_{t=1}^\infty \1\{ I_t \in \H_1 \}  &\leq  \max_\pi \sum_{i \in \H_1} \phi^{-1}(\tfrac{\mu_i-\mu_0}{2},\delta'\rho_i\tfrac{\pi(i)}{n})\\
&\leq \max_\pi \sum_{i \in \H_1} \left(c \Delta_i^{-2} \log ( \tfrac{n}{\pi(i)}\log(\Delta_i^{-2}) / \delta') + c \Delta_i^{-2} \log (1/\rho_i) \right)  \\
&\overset{\mathcal{E}_{4,1}}{\leq} \max_\pi \sum_{i \in \H_1} \left(c \Delta_i^{-2} \log ( \tfrac{n}{\pi(i)}\log(\Delta_i^{-2}) / \delta) + 5 c \Delta_i^{-2} \log (1/\delta) \right) \\
&\leq \max_\pi \sum_{i \in \H_1} c' \Delta_i^{-2} \log ( \tfrac{n}{\pi(i)}\log(\Delta_i^{-2}) / \delta) \\
&= \sum_{i=1}^{|\H_1|} c' \Delta_i^{-2} \log ( \tfrac{n}{i}\log(\Delta_i^{-2}) / \delta).
\end{align*}
The first claimed upper bound of $T$, the diverse means case, is completed by considering the second to last line and noting trivially that $\pi(i)\geq 1$.
To obtain the second upper bound of $T$, we consider the last line and note that
\begin{align*}
\sum_{i=1}^k \log(\tfrac{n}{i}) &\leq \int_{0}^k \log(\tfrac{n}{x}) dx = k \log(n) - (x \log x - x)\Big|_{x=0}^k = k \log(\tfrac{n}{k}) + k \leq n.
\end{align*}
Combining the previous two displays we get
\begin{align*}
\sum_{t=1}^\infty \1\{ I_t \in \H_1 \} &\leq \sum_{i=1}^{|\H_1|} c \Delta^{-2} \log ( \tfrac{n}{i}\log(\Delta^{-2}) / \delta) \\
&\leq c |\H_1| \Delta^{-2} \log ( \log(\Delta^{-2}) / \delta) + \sum_{i=1}^{|\H_1|} c \Delta^{-2} \log ( \tfrac{n}{i}) \\
&\leq c |\H_1| \Delta^{-2} \log ( \log(\Delta^{-2}) / \delta) +  c n \Delta^{-2}  \\
&\leq c'' n \Delta^{-2} \log ( \log(\Delta^{-2}) / \delta) .
\end{align*}
\end{proof}

\section{Proof of Theorem~\ref{thm:FDR_FWPD}}\label{sec:FDR_FWPD_proof}
 Our analysis will also make use of the following events, that we will prove each hold with probability at least $1-\delta$. Let $\beta := \tfrac{5}{3(1-4\delta)} \log(1/\delta)$ and define:
\begin{align*}
\mathcal{E}_1 &:= \left\{ \left| \left\{ i \in \H_1 : \bigcap_{t=1}^\infty \{ |\widehat{\mu}_{i,t} - \mu_i| \leq \phi(t,\tfrac{\delta}{\beta}) \} \right\} \right| \geq \frac{1}{2} |\H_1|  \right\} \\
\mathcal{E}_{2, j} &:= \left\{ \bigcap_{i \in \H_j} \bigcap_{t=1}^\infty \{ |\widehat{\mu}_{i,t} - \mu_i| \leq \phi(t, \tfrac{\delta}{|\H_j|}) \}  \right\}\quad j\in \{0,1\}
\end{align*}
Also, recall from Lemma \ref{lem:fdr_high_prob} the event
\begin{align*}
\mathcal{E}_3 &:= \bigg\{ \forall s \in (0,1]: \sum_{i \in \H_0} \1\{P_{i,\ast}\leq s\} \\
    &\hspace{.5in}\leq  s |\H_0| + (1+4s) \sqrt{2 \max\{2s, 2\delta'/n\} |\H_0| \log(\tfrac{\log_2(n/\delta')}{\delta'}) } + \tfrac{1+4s}{3} \log(\tfrac{\log_2(n/\delta')}{\delta'}) \bigg\}
\end{align*}
Lemma \ref{lem:fdr_high_prob} guarantees that this event holds with probability greater than $1-\delta$.

We restate the theorem using the above events.
\begin{reptheorem}{thm:FDR_FWPD}[FDR, FWPD]
For all $t$ we have $\E[\frac{|\calS_t \cap \H_0|}{|\calS_t|}] \leq \delta$.
Moreover, on $\mathcal{E}_1 \cap \mathcal{E}_{2,1} \cap \mathcal{E}_3 \cap \mathcal{E}_{4,0} \cap \mathcal{E}_{4,1}$ (which holds with probability at least $1-5\delta$), there exists a $T$ such that
\begin{align*}
T &\lesssim  (n-|\H_1|) \Delta^{-2} \log (  \max\{|\H_1|, \log\log(n/\delta)\} \log(\Delta^{-2}) / \delta)  + |\H_1| \Delta^{-2} \log ( \log(\Delta^{-2}) / \delta)
\end{align*}
and $\H_1 \subseteq \calS_t$ for all $t \geq T$.
\end{reptheorem}

We need a few technical lemmas, specifically, the proof of events $\mathcal{E}_1$ and $\mathcal{E}_{2,1}$ and their consequences.

\begin{lemma}
$\P(\mathcal{E}_1) \geq 1-\delta$.
\end{lemma}
\begin{proof}
We break the proof up into two cases based on the cardinality $|\H_1|$.
If $|\H_1| \leq \beta$ then
\begin{align*}
 \P\left( \bigcup_{i\in\H_1} \bigcup_{t=1}^\infty \{ |\widehat{\mu}_{i,t} - \mu_i| \geq \phi(t, \tfrac{\delta}{ \beta }) \}\right) \leq \sum_{i\in\H_1}\P\left( \bigcup_{t=1}^\infty \{ |\widehat{\mu}_{i,t} - \mu_i| \geq \phi(t, \tfrac{\delta}{ \beta }) \}\right) \leq \delta \tfrac{|\H_1|}{\beta}
\end{align*}
which is less than $\delta$ by the case definition.
So in what follows, assume that $|\H_1| > \beta$.
By definition
\begin{align*}
\P\left( \bigcup_{t=1}^\infty \{ |\widehat{\mu}_{i,t} - \mu_i| \geq \phi(t, \delta) \} \right)
&= \P(\rho_i \leq \delta) \leq \delta.
\end{align*}
By Bernstein's inequality, with probability at least $1-\delta$
\begin{align*}
\sum_{i \in \H_1} \1\{ \rho_i \leq \delta \} &\leq \delta  |\H_1|  + \sqrt{ 2 \delta  |\H_1| \log(1/\delta) } + \tfrac{1}{3} \log(1/\delta)\\
&\leq \delta  |\H_1|  + 2 \sqrt{ \tfrac{1}{2} \delta  |\H_1| \log(1/\delta) } + (1-\tfrac{1}{3}) \tfrac{1}{2} \log(1/\delta) \\
&\leq 2 \delta |\H_1| +  \tfrac{5}{6} \log(1/\delta)
\end{align*}
where the last line follows from $a + 2 \sqrt{ab} + b = (\sqrt{a}+\sqrt{b})^2 \leq 2a + 2b$,
which implies
\begin{align*}
\sum_{i \in \H_1} \1\{ \rho_i > \delta \} &\geq (1-2\delta) |\H_1| -  \tfrac{5}{6} \log(1/\delta) \\
&\geq \tfrac{1}{2} |\H_1|
\end{align*}
where we use the fact that $|\H_1| > \beta = \tfrac{5}{3(1-4\delta)} \log(1/\delta)$.
Combining these two results, and noting that at most one of the cases $|\H_1| > \beta$ or $|H_1| \leq \beta$ can be true, we obtain that $\P(\mathcal{E}_1) \geq 1-\delta$.
\end{proof}

\begin{lemma}
$\min\{ \P(\mathcal{E}_{2,0}), \P(\mathcal{E}_{2,1}) \} \geq 1-\delta$.
\end{lemma}
\begin{proof}
The result follows from a union bound:
\begin{align*}
\P(\mathcal{E}_{2,1}^c) &= \P\left(\left\{ \bigcup_{i \in \H_1} \bigcup_{t=1}^\infty \{ |\widehat{\mu}_{i,t} - \mu_i| \leq \phi(t, \tfrac{\delta}{|\H_1|}) \}  \right\}\right)\\
&\leq \sum_{i \in \H_1} \P\left(\left\{ \bigcup_{t=1}^\infty \{ |\widehat{\mu}_{i,t} - \mu_i| \leq \phi(t, \tfrac{\delta}{|\H_1|}) \}  \right\}\right) \leq \sum_{i \in \H_1} \frac{\delta}{|\H_1|} \leq \delta.
\end{align*}
The proof that $\P(\mathcal{E}_{2,0}) \geq 1-\delta$ follows analogously.
\end{proof}

The next lemma shows an important consequence of these events holding.

\begin{lemma}\label{lem:H1_ucb_highest}
If $\mathcal{E}_1 \cap \mathcal{E}_{2,1}$ then for all $t$
\begin{align*}
\H_1 \nsubseteq \calS_t \implies \max_{i \in \H_1 \cap \calS_t^c} \widehat{\mu}_{i,T_i(t)} + \phi(T_i(t), \tfrac{\delta}{2|\calS_t| \vee \beta }) \geq \min_{i \in \H_1} \mu_i.
\end{align*}
\end{lemma}
\begin{proof}
Define the random set $\mathcal{I}_t = \left\{ i \in \H_1 : \widehat{\mu}_{i,T_i(t)} + \phi(T_i(t), \tfrac{\delta}{2|\calS_t| \vee \beta }) \geq \mu_i \right\}$.
We will prove that on $\mathcal{E}_1 \cap \mathcal{E}_{2,1}$ we have $\mathcal{I}_t \cap \calS_t^c \neq \emptyset$ which implies the result.
First we use the fact that $2|\calS_t| \vee \beta  \geq \beta$ so that
\begin{align*}
|\mathcal{I}_t| &= \left| \left\{ i \in \H_1 : \widehat{\mu}_{i,T_i(t)} + \phi(T_i(t), \tfrac{\delta}{2|\calS_t| \vee \beta }) \geq \mu_i \right\} \right|\\
&\geq \left| \left\{ i \in \H_1 : \widehat{\mu}_{i,T_i(t)} + \phi(T_i(t), \tfrac{\delta}{\beta }) \geq \mu_i \right\} \right|\\
&\geq \left| \left\{ i \in \H_1 : \bigcap_{t=1}^\infty \{ |\widehat{\mu}_{i,T_i(t)} - \mu_i| \leq \phi(T_i(t), \tfrac{\delta}{\beta }) \} \right\} \right|	\\
&\geq \left| \left\{ i \in \H_1 : \bigcap_{t=1}^\infty \{ |\widehat{\mu}_{i,t} - \mu_i| \leq \phi(t, \tfrac{\delta}{\beta }) \} \right\} \right|	\\
&\overset{\mathcal{E}_1}{\geq} \frac{1}{2}|\H_1| .
\end{align*}
Given $|\mathcal{I}_t| \geq \frac{1}{2}|\H_1|$,  if $|\calS_t| < \frac{1}{2}|\H_1|$ then $|\calS_t| < |\mathcal{I}_t|$ which implies $\calS_t^c \cap \mathcal{I}_t \neq \emptyset$.
On the other hand, if $|\calS_t| \geq |\H_1|/2$ then we use the fact that $2|\calS_t| \vee \beta  \geq 2 |\calS_t| \geq |\H_1|$ to observe
\begin{align*}
|\mathcal{I}_t| &= \left| \left\{ i \in \H_1 : \widehat{\mu}_{i,T_i(t)} + \phi(T_i(t), \tfrac{\delta}{2|\calS_t| \vee \beta }) \geq \mu_i \right\} \right|\\
&\geq \left| \left\{ i \in \H_1 : \widehat{\mu}_{i,T_i(t)} + \phi(T_i(t), \tfrac{\delta}{|\H_1| }) \geq \mu_i \right\} \right|\\
&\geq \left| \left\{ i \in \H_1 : \bigcap_{t=1}^\infty \{ |\widehat{\mu}_{i,T_i(t)} - \mu_i| \leq \phi(T_i(t), \tfrac{\delta}{|\H_1|}) \} \right\} \right|	\\
&\geq \left| \left\{ i \in \H_1 : \bigcap_{t=1}^\infty \{ |\widehat{\mu}_{i,t} - \mu_i| \leq \phi(t, \tfrac{\delta}{|\H_1|}) \} \right\} \right|	\\
&\overset{\mathcal{E}_{2,1}}{\geq} |\H_1|
\end{align*}
which implies $\mathcal{I}_t=\H_1$, thus $\mathcal{I}_t \cap \calS_t^c = \H_1 \cap \calS_t^c$ which is non-empty by assumption.
\end{proof}

We are now ready to prove Theorem~\ref{thm:FDR_FWPD}.
\begin{proof}
We proceed similarly to Theorem 1.
Note that $\calS_t \subseteq \calS_{t+1}$ for all $t$ so define $T = \min \{ t \in \mathbb{N} : \H_1 \subseteq \calS_{t+1} \}$ if such inclusion ever exists, otherwise let $T = \infty$.
Then
\begin{align*}
T &= \sum_{t=1}^\infty \1\{ \H_1 \not\subseteq \calS_t \} \\
&= \sum_{t=1}^\infty \1\{ I_t \in \H_0, \H_1 \not\subseteq \calS_t\} + \1\{ I_t \in \H_1\} .
\end{align*}
Note that we are in the TPR setting, like Theorem~\ref{thm:FDR_TPR}, and so the upperbound $\sum_{t=1}^\infty \1\{ I_t \in \H_1\} \leq c n\Delta^{-2} \log( \log(\Delta^{-2}) /\delta)$ applies here as well.
Thus, we only need to bound the first sum.

Lemma~\ref{lem:H1_ucb_highest} (which requires $\mathcal{E}_1 \cap \mathcal{E}_{2,1}$) says that if there is at least one arm from $\H_1$ not in $\calS_t$ then the largest upper confidence bound of some arm in $\H_1 \cap \calS_t^c$ is \emph{at least} as large as $\min_{j \in \H_1} \mu_j \geq \mu_0 + \Delta$.
Thus, $\{\H_1 \not\subseteq \calS_t\}$ implies that
\begin{align*}
\arg\max_{i \in \S_t^c} \widehat{\mu}_{i,T_i(t)} + \phi(T_i(t),\delta)  \geq \min_{i \in \H_1} \mu_i = \mu_0 + \Delta.
\end{align*}

On the other hand, for $\kappa= \tfrac{1}{1-2\delta'(1+4\delta')} |\H_1| + \tfrac{4(1+4\delta')/3}{1-2\delta'(1+4\delta')} \log(\tfrac{5\log_2(n/\delta')}{\delta'})$, we have $|\S_t| \leq \kappa$ from Lemma~\ref{lem:fdr_high_prob} (which requires $\mathcal{E}_3$),
and for any $i \in \H_0$ we have
\begin{align*}
\widehat{\mu}_{i,T_i(t)} + \phi(T_i(t),\tfrac{\delta}{|\S_t| \vee \beta}) &\leq \mu_i + \phi(T_i(t),\tfrac{\delta}{|\S_t| \vee \beta}) + \phi(T_i(t),\rho_i) \\
&\leq \mu_i + \phi(T_i(t),\tfrac{\delta}{\kappa}) + \phi(T_i(t),\rho_i) \\
&\leq \mu_i + 2 \phi(T_i(t),\tfrac{\delta \rho_i}{\kappa})\\
&\leq \mu_0 + 2 \phi(T_i(t),\tfrac{\delta \rho_i}{\kappa})
\end{align*}
so that $\widehat{\mu}_{i,T_i(t)} + \phi(T_i(t),\tfrac{\delta}{|\S_t| \vee \beta}) \leq \mu_0 + \Delta$ whenever $T_i(t) \geq \phi^{-2}(\tfrac{\Delta}{2}, \tfrac{\delta \rho_i}{\kappa})$.
Thus,
\begin{align*}
\sum_{t=1}^\infty \1\{ I_t \in \H_0, \H_1 \not\subseteq \calS_t \} &\leq \sum_{i \in \H_0} \phi^{-2}(\tfrac{\Delta}{2}, \tfrac{\delta \rho_i}{\kappa}) \\
&\leq \sum_{i \in \H_0} c \Delta^{-2} \log ( \kappa \log(\Delta^{-2}) / \delta) + c \Delta^{-2} \log(1/\rho_i) \\
&\overset{\mathcal{E}_{4,0}}{\leq} \sum_{i \in \H_0} c' \Delta^{-2} \log ( \max\{|\H_1|, \log\log(n/\delta)\} \log(\Delta^{-2}) / \delta) + 5 c \Delta^{-2} \log(1/\delta) \\
&\leq |\H_0| c'' \Delta^{-2} \log ( \max\{|\H_1|, \log\log(n/\delta)\} \log(\Delta^{-2}) / \delta)
\end{align*}
We conclude that
\begin{align*}
T &= \sum_{t=1}^T \1\{ I_t \in \H_0, \H_1 \not\subseteq \calS_t\} + \1\{ I_t \in \H_1, \H_1 \not\subseteq \calS_t\} \\
&\leq (n-|\H_1|) c \Delta^{-2} \log ( \max\{|\H_1|, \log\log(n/\delta)\} \log(\Delta^{-2}) / \delta) + c n \Delta^{-2} \log ( \log(\Delta^{-2}) / \delta)\\
&\leq (n-|\H_1|) c' \Delta^{-2} \log ( \max\{|\H_1|, \log\log(n/\delta)\} \log(\Delta^{-2}) / \delta) + c' |\H_1| \Delta^{-2} \log ( \log(\Delta^{-2}) / \delta).
\end{align*}
\end{proof}

\section{Proof of Theorem~\ref{thm:FWER_FWPD}}
We restate the theorem using the above events.
\begin{reptheorem}{thm:FWER_FWPD}[FWER, FWPD]
For all $t$ we have $\E[\frac{|\calS_t \cap \H_0|}{|\calS_t|}] \leq \delta$.
Moreover, on $\mathcal{E}_1 \cap \mathcal{E}_{2,0} \cap \mathcal{E}_{2,1} \cap \mathcal{E}_3 \cap \mathcal{E}_{4,0} \cap \mathcal{E}_{4,1} $ (which holds with probability at least $1-6\delta$), we have $\H_0 \cap \calR_t = \emptyset$ for all $t \in \mathbb{N}$ and there exists a $T$ such that
\begin{align*}
T \lesssim&  (n-|\H_1|) \Delta^{-2} \log (  \max\{|\H_1|, \log\log(n/\delta)\} \log(\Delta^{-2}) / \delta) \\
&+ |\H_1| \Delta^{-2} \log ( \max\{n-(1-2\delta(1+4\delta))|\H_1|, \log\log(n/\delta)\} \log(\Delta^{-2}) / \delta)
\end{align*}
and $\H_1 \subseteq \calR_t$ for all $t \geq T$.
Note, together this implies $\H_1 = \calR_t$ for all $t \geq T$.
\end{reptheorem}

The following lemma shows that a tight control on the size of $\calS_t$ allows us to conclude a FWER.

\begin{lemma}\label{lem:fwer}
If $\mathcal{E}_3 \cap \mathcal{E}_{2,0}$ holds then $\calR_t \cap \H_0 = \emptyset$ for all $t$.
\end{lemma}
\begin{proof}
By Lemma~\ref{lem:fdr_high_prob} (which requires $\mathcal{E}_3$) we have $|\S_t| \leq \tfrac{1}{1-2\delta'(1+4\delta')} |\H_1| + \tfrac{4(1+4\delta')/3}{1-2\delta'(1+4\delta')} \log(\tfrac{5\log_2(n/\delta')}{\delta'}) = \frac{|\H_1| + \eta}{1-2\delta'(1+4\delta')}$ for all times $t$ and $\eta = \tfrac{4(1+4\delta')}{3}\log(5\log_2(n)/\delta')$.
This implies
\begin{align} \label{eqn:H0_lessthan}
n - (1-2\delta'(1+4\delta')) |\S_t| + \eta \geq n - |\H_1| = |\H_0|
\end{align}
but for $i \in \S_t \cap \H_0 \subseteq \H_0$ we have that $\mathcal{E}_{2,0}$ applies so
\begin{align*}
\widehat{\mu}_{i,T_i(t)} - \phi(T_i(t),\tfrac{\delta}{n - (1-2\delta'(1+4\delta')) |\S_t| + \eta}) &\leq \widehat{\mu}_{i,T_i(t)} - \phi(T_i(t),\tfrac{\delta}{|\H_0|}) \overset{\mathcal{E}_{2,0}}{\leq} \mu_i \leq \mu_0
\end{align*}
where the last inequality holds because $\max_{i \in \H_0} \mu_i \leq \mu_0$.
Thus, no arms from $\H_0$ will be added to $\calR_t$.
\end{proof}

Now that we have FWER control, we need to show that all the arms in $\H_1$ are added to $\calR_t$ in the claimed amount of time.
\begin{proof}
Note that $\calR_t \subseteq \calR_{t+1}$ for all $t$ so define $T = \min \{ t \in \mathbb{N} : \H_1 \subseteq \calR_{t+1} \}$ if such inclusion ever exists, otherwise let $T = \infty$.
Noting that $\calR_t \subseteq \calS_t$ we have
\begin{align*}
T &= \sum_{t=1}^\infty \1\{ \H_1 \not\subseteq \calR_t\} \\
&= \sum_{t=1}^\infty \1\{ \H_1 \not\subseteq \calS_t\} + \1\{ \H_1 \not\subseteq \calR_t, \H_1 \subseteq \calS_t \}\\
&= \sum_{t=1}^\infty \1\{ \H_1 \not\subseteq \calS_t\} + \1\{ J_t \in \H_0, \H_1 \not\subseteq \calR_t, \H_1 \subseteq \calS_t \} + \1\{ J_t \in \H_1, \H_1 \not\subseteq \calR_t, \H_1 \subseteq \calS_t \}
\end{align*}
\textbf{First sum.} Note that we are in the FWPD setting, so the selection rule for $I_t$ identical to that of the setting of Theorem~\ref{thm:FDR_FWPD},
and $\sum_{t=1}^\infty \1\{ \H_1 \not\subseteq \calS_t \}$ is precisely what is bounded in the proof of Theorem~\ref{thm:FDR_FWPD} (which requires $\mathcal{E}_1 \cap \mathcal{E}_{2,1} \cap \mathcal{E}_3 \cap \mathcal{E}_{4,0} \cap \mathcal{E}_{4,1}$).
Thus,
\begin{align*}
&\sum_{t=1}^T \1\{ \H_1 \not\subseteq \calS_t \} \\
&\leq (n-|\H_1|) c' \Delta^{-2} \log ( \max\{|\H_1|, \log\log(n/\delta)\} \log(\Delta^{-2}) / \delta) + c' |\H_1| \Delta^{-2} \log ( \log(\Delta^{-2}) / \delta).
\end{align*}

\textbf{Second sum.} Recall that $\displaystyle J_t = \arg\min_{i \in \S_t\setminus \calR_t} \widehat{\mu}_{i,T_i(t)} + \phi(T_i(t),\tfrac{\delta}{|\calS_t|})$.
Now, on the event $\{ \H_1 \not\subseteq \calR_t, \H_1 \subseteq \calS_t \}$ we have that $|\H_1| \leq |\calS_t|$ and that there exists a $j \in \H_1 \cap (\calS_t \setminus \calR_t)$ such that
\begin{align*}
\widehat{\mu}_{j,T_j(t)} + \phi(T_j(t), \tfrac{\delta}{|\S_t|}) &\geq \widehat{\mu}_{j,T_j(t)} + \phi(T_j(t), \tfrac{\delta}{|\H_1|}) \\
&\overset{\mathcal{E}_{2,1}}{\geq} \mu_j \\
&\geq \mu_0 + \Delta.
\end{align*}
On the other hand, for $\kappa=\tfrac{1}{1-2\delta'(1+4\delta')} |\H_1| + \tfrac{4(1+4\delta')/3}{1-2\delta'(1+4\delta')} \log(\tfrac{5\log_2(n/\delta')}{\delta'})$, we have $|\S_t| \leq \kappa$ from Lemma~\ref{lem:fdr_high_prob} (which requires $\mathcal{E}_3$), and for any $i \in \H_0 \cap (\S_t \setminus \calR_t)$ we have
\begin{align*}
\widehat{\mu}_{i,T_i(t)} + \phi(T_i(t),\tfrac{\delta}{|\S_t|})
&\leq \mu_i + \phi(T_i(t),\tfrac{\delta}{\kappa}) + \phi(T_i(t),\rho_i) \\
&\leq \mu_i + 2 \phi(T_i(t),\tfrac{\delta \rho_i}{\kappa}) \\
&\leq \mu_0 + 2 \phi(T_i(t),\tfrac{\delta \rho_i}{\kappa})
\end{align*}
so that $\widehat{\mu}_{i,T_i(t)} + \phi(T_i(t),\tfrac{\delta}{|\S_t|}) \leq \mu_0 + \Delta$ whenever $T_i(t) \geq \phi^{-2}(\tfrac{\Delta}{2}, \tfrac{\delta \rho_i}{\kappa})$.
By an identical argument to that made in the proof of Theorem~\ref{thm:FDR_FWPD} we very crudely have the bound
\begin{align*}
\sum_{t=1}^\infty  \1\{ J_t \in \H_0, \H_1 \not\subseteq \calR_t, \H_1 \subseteq \calS_t \}  &\leq \sum_{i \in \H_0} \phi^{-2}(\tfrac{\Delta}{2}, \tfrac{\delta \rho_i}{\kappa})\\
&\leq c |\H_0| \Delta^{-2} \log ( \max\{|\H_1|, \log\log(n/\delta)\} \log(\Delta^{-2}) / \delta).
\end{align*}

\textbf{Third sum.}
An arm $j \in \H_1 \cap (\calS_t \setminus \calR_t)$ is accepted into $\calR_t$ if $\widehat{\mu}_{i,T_i(t)} - \phi(T_i(t), \tfrac{\delta}{\chi_t}) \geq \mu_0$ where $\chi_t = n - (1-2\delta'(1+4\delta')) |\S_t| + \tfrac{4(1+4\delta')}{3}\log(5\log_2(n/\delta')/\delta')$.
On the event $\{\H_1 \not\subseteq \calR_t, \H_1 \subseteq \calS_t \}$ we have $\chi_t \leq n - (1-2\delta'(1+4\delta')) |\H_1| + \tfrac{4(1+4\delta')}{3}\log(5\log_2(n/\delta')/\delta') =: u$.
Thus, for $j \in \H_1 \cap (\calS_t \setminus \calR_t)$
\begin{align*}
\widehat{\mu}_{j,T_j(t)} - \phi(T_j(t), \tfrac{\delta}{\chi_t}) &\geq \widehat{\mu}_{j,T_j(t)} - \phi(T_j(t), \tfrac{\delta}{u}) \\
&\geq \mu_j - \phi(T_j(t), \tfrac{\delta}{u}) - \phi(T_j(t),\rho_j) \\
&\geq \mu_j - 2\phi(T_j(t), \tfrac{\delta \rho_j}{u})\\
&\geq \mu_0 + \Delta - 2\phi(T_j(t), \tfrac{\delta \rho_j}{u})
\end{align*}
so that $\widehat{\mu}_{j,T_j(t)} - \phi(T_j(t), \tfrac{\delta}{\chi_t}) \geq \mu_0$ whenever $T_j(t) \geq \phi^{-1}(\tfrac{\Delta}{2},\tfrac{\delta \rho_j}{u})$.
By the same arguments made throughout these proofs we have
\begin{align*}
\sum_{t=1}^\infty \1\{ J_t \in \H_1, \H_1 \not\subseteq \calR_t, \H_1 \subseteq \calS_t \}  &\leq \sum_{j \in \H_1} \phi^{-1}(\tfrac{\Delta}{2},\tfrac{\delta \rho_j}{u}) \\
&\hspace{-1in}\leq c |\H_1| \Delta^{-2} \log ( u \log(\Delta^{-2}) / \delta) \\
&\hspace{-1in}\leq c' |\H_1| \Delta^{-2} \log ( \max\{n-(1-2\delta(1+4\delta))|\H_1|, \log\log(n/\delta)\} \log(\Delta^{-2}) / \delta).
\end{align*}
Summing all three sums together yields the result.
\end{proof}

\section{Proof of Theorem~\ref{thm:FWER_TPR}}
We define a new event,
\begin{align*}
\mathcal{E}_5 &:= \left\{ \sum_{i \in \H_1} \1\{ \rho_i \leq \delta \} \leq \delta |\H_1| + \sqrt{2\delta |\H_1| \log(1/\delta)} + \tfrac{1}{3}\log(1/\delta)\right\}
\end{align*}
By a direct application of Bernstein's inequality, this holds with probability greater than $1-\delta$.

\begin{reptheorem}{thm:FWER_TPR}[FWER, TPR]
For all $t$ we have $\E[\frac{|\calS_t \cap \H_0|}{|\calS_t|}] \leq \delta$.
Moreover, on $\mathcal{E}_1 \cap \mathcal{E}_{2,1} \cap \mathcal{E}_3 \cap \mathcal{E}_{4,0} \cap \mathcal{E}_{4,1} \cap \mathcal{E}_{2,0} \cap \mathcal{E}_5$ (which holds with probability at least $1-7\delta$), we have $\H_0 \cap \calR_t = \emptyset$ for all $t \in \mathbb{N}$ and there exists a $T$ such that
\begin{align*}
T \lesssim&  (n-|\H_1|) \Delta^{-2} \log (  \log(\Delta^{-2}) / \delta) \\
&+ |\H_1| \Delta^{-2} \log ( \max\{n-(1-2\delta)|\H_1|, \log\log(n/\delta)\} \log(\Delta^{-2}) / \delta)
\end{align*}
and $\E[\frac{|\calR_t \cap \H_1|}{|\H_1|}] \geq 1-\delta$ for all $t \geq T$.
\end{reptheorem}

\begin{proof}
First, we invoke Lemma~\ref{lem:fwer} (which requires $\mathcal{E}_3 \cap \mathcal{E}_{2,0}$) which controls the FWER.
All that is left is to control the sample complexity.

Let $\calI = \{ i \in \H_1 : \rho_i \geq \delta \}$ be the same random set defined in the proof of Theorem~\ref{thm:FDR_TPR}.
Note that $\calR_t \subseteq \calR_{t+1}$ for all $t$ so define $T = \min \{ t \in \mathbb{N} : \calI \subseteq \calR_{t+1} \}$ if such inclusion ever exists, otherwise let $T = \infty$.
Noting that $\calR_t \subseteq \calS_t$ we have
\begin{align*}
T &= \sum_{t=1}^\infty \1\{ \calI \not\subseteq \calR_t\} \\
&= \sum_{t=1}^\infty \1\{ \calI \not\subseteq \calS_t\} + \1\{ \calI \not\subseteq \calR_t, \calI \subseteq \calS_t\}\\
&= \sum_{t=1}^\infty \1\{ \calI \not\subseteq \calS_t\} + \1\{ J_t \in \H_0, \calI \not\subseteq \calR_t, \calI \subseteq \calS_t\} + \1\{ J_t \in \H_1, \calI \not\subseteq \calR_t, \calI \subseteq \calS_t\}.
\end{align*}

\textbf{First sum.} Since we are in the TPR setting we have $\xi_t=1$ so the first sum is precisely the quantity bounded in the proof of Theorem~\ref{thm:FDR_TPR}.
Thus,
\begin{align*}
\sum_{t=1}^\infty \1\{ \calI \not\subseteq \calS_t\} \leq c n \Delta^{-2} \log (  \log(\Delta^{-2}) / \delta).
\end{align*}

\textbf{Second sum.}
Recall that $\displaystyle J_t = \arg\max_{i \in \S_t\setminus \calR_t} \widehat{\mu}_{i,T_i(t)} + \phi(T_i(t),\delta)$.
Now, on the event $\{ \calI \not\subseteq \calR_t, \calI \subseteq \calS_t \}$ there exists a $j \in \calI \cap (\calS_t \setminus \calR_t)$ such that
\begin{align*}
\widehat{\mu}_{j,T_j(t)} + \phi(T_j(t), \delta) &\geq \mu_j + \phi(T_j(t), \delta) - \phi(T_j(t), \rho_j)\\
&\geq \mu_j \\
&\geq \mu_0 + \Delta
\end{align*}
where the first inequality follows by the definition of $\calI$.
On the other hand, for any $i \in \H_0$ we have
\begin{align*}
\widehat{\mu}_{i,T_i(t)} + \phi(T_i(t),\delta) &\leq \mu_i + \phi(T_i(t),\delta) + \phi(T_i(t),\rho_i) \\
&\leq \mu_i + 2 \phi(T_i(t),\delta \rho_i) \\
&\leq \mu_0 + 2 \phi(T_i(t),\delta \rho_i)
\end{align*}
so that $\widehat{\mu}_{i,T_i(t)} + \phi(T_i(t),\delta) \leq \mu_0 + \Delta$ whenever $T_i(t) \geq \phi^{-1}(\tfrac{\Delta}{2}, \delta \rho_i)$.
Thus, by the same sequence of steps used in Theorem~\ref{thm:FDR_TPR} we have
\begin{align*}
\sum_{t=1}^\infty \1\{ J_t \in \H_0, \calI \not\subseteq \calR_t, \calI \subseteq \calS_t\} &\leq \sum_{i \in \H_0} \phi^{-1}(\tfrac{\Delta}{2}, \delta \rho_i) \overset{\mathcal{E}_{4,0}}{\leq} \ c |\H_0| \Delta^{-2} \log ( \log(\Delta^{-2}) / \delta).
\end{align*}

\textbf{Third sum.} In bounding the analogous sum in the proof of Theorem~\ref{thm:FWER_FWPD} we used the fact that $\H_1 \subseteq \calS_t$ to lowerbound $|\calS_t|$ to obtain an upperbound on $\xi_t$.
Now that we only have $\calI \subseteq \calS_t$ we observe that
\begin{align*}
|\calS_t| \geq |\calI| \geq |\H_1| - \sum_{i \in \H_1} \1\{ \rho_i \leq \delta \} \overset{\mathcal{E}_5}{\geq} |\H_1|( 1 - \delta- \sqrt{\tfrac{2\delta \log(1/\delta)}{|\H_1|}} - \tfrac{\log(1/\delta)}{3|\H_1|}).
\end{align*}
Using this approximation the same argument yields
\begin{align*}
&\sum_{t=1}^\infty \1\{ J_t \in \H_1, \calI \not\subseteq \calR_t, \calI \subseteq \calS_t\} \\
&\leq c' |\H_1| \Delta^{-2} \log ( \max\{n-( 1 - 3\delta- \sqrt{\tfrac{2\delta \log(1/\delta)}{|\H_1|}} - \tfrac{\log(1/\delta)}{3|\H_1|})|\H_1|, \log\log(n/\delta)\} \log(\Delta^{-2}) / \delta) \\
&\leq c'' |\H_1| \Delta^{-2} \log ( \max\{n-( 1 - 3\delta- \sqrt{2\delta \log(1/\delta)/|\H_1|}) |\H_1|, \log\log(n/\delta)\} \log(\Delta^{-2}) / \delta).
\end{align*}
\end{proof}

\section{Successive Elimination and Uniform Allocation Algorithms}\label{sec:succ-elim}
\begin{algorithm}[h]
 	\textbf{Input:} Threshold $\mu_0$, confidence $\delta$, confidence interval $\phi(\cdot,\cdot)$\\
 	\textbf{Initialize:} Set $\S_{1} = \emptyset$

	\textbf{For} $t = 1,2,\dots$ \\
    \hspace*{.25in} \textbf{if Successive Elimination}\\
	   \hspace*{.50in}\textbf{Pull each and every arm} in $[n] - \S_t$.\\
    \hspace*{.25in} \textbf{else if Uniform}\\
        \hspace*{.50in}\textbf{Pull each and every arm} in $[n]$.\\
	\hspace*{.25in}\textbf{Apply} Benjamini-Hochberg \cite{benjamini1995controlling} selection to obtain FDR-controlled set $\calS_t$:\\[2pt]
	\hspace*{.5in}$s(k) = \{ i \in [n]\setminus \calS_t: \widehat{\mu}_{i,t} - \phi(t,\delta \tfrac{k}{n}) \geq \mu_0 \}$, $\forall k \in [n]$ \\
	\hspace*{.5in}$\displaystyle\S_{t+1} = \S_t \cup s(\widehat{k}) \text{ where }\widehat{k} = \max \{ k \in [n]: |s(k)| \geq k \}$ (if $\not\exists \widehat{k}$ set $\S_{t+1} = \S_t$)\\[4pt] 

	\caption{\small Uniform and Successive elimination algorithms for identifying arms with means above $\mu_0$.
    \label{alg:SE}}
\end{algorithm}
The following gives a sample complexity bound for Successive Elimination and Uniform allocation strategies.
Note that for these algorithms, up to $n$ arms are pulled at each time $t$.

\begin{theorem}
For both the Successive Elimination and Uniform Allocation algorithms, for all $t$ we have that $\E[\frac{|\calS_t \cap \H_0|}{|\calS_t|}] \leq \delta$. In addition, in the case of successive elimination, if $\eta = \delta + \sqrt{\tfrac{2\delta \log(1/\delta)}{|\H_1|}} + \tfrac{\log(1/\delta)}{3|\H_1|}$, then on the event $\mathcal{E}_5$ (which holds with probability greater than $1-\delta$) then there exists a $T$ such that
\begin{align*}
T \lesssim \min\Big\{ & n \Delta^{-2} \log(\tfrac{n}{(1-\eta)|\H_1|} \log(\Delta^{-2})/\delta) ,\\
& (n-(1-\eta)|\H_1|) \Delta^{-2} \log(\tfrac{n}{(1-\eta)|\H_1|} \log(\Delta^{-2})/\delta) + \sum_{i \in \H_1}   \Delta_i^{-2} \log(n  \log(\Delta_i^{-2})/\delta) \Big\}
\end{align*}
and in the case of Uniform allocation,
\begin{align*}
T \lesssim n \Delta^{-2} \log(\tfrac{n}{(1-\eta)|\H_1|} \log(\Delta^{-2})/\delta)
\end{align*}
and $\E[\frac{|\calS_t \cap \H_1|}{|\calS_t|}] > 1-\delta$ for all $t\geq T$.
\end{theorem}
\begin{proof}
Define the random set $\calI = \{ i \in \H_1 : \rho_i \geq \delta \}$.
Note, this is equivalent to $\calI=\{ i \in \H_1 : \widehat{\mu}_{i,T_i(t)} + \phi(T_i(t),\delta) \geq \mu_i \ \ \forall t \in \mathbb{N} \}$ since if $\rho_i \geq \delta$ then $\phi(T_i(t),\delta) > \phi(T_i(t),\rho_i)$ and
\begin{align*}
\widehat{\mu}_{i,T_i(t)} + \phi(T_i(t),\delta) \geq \mu_i + \phi(T_i(t),\delta) - \phi(T_i(t),\rho_i) \geq \mu_i.
\end{align*}
Our aim is to show that this ``well-behaved'' set of indices $\calI$ will be added to $\S_t$ in the claimed amount of time.
This is sufficient for the TPR result because $\E|\calI| = \sum_{i \in \H_1} \E[\1\{ \rho_i > \delta\}] = \sum_{i \in \H_1} \P( \rho_i > \delta) \geq (1-\delta) |\H_1|$ since each $\rho_i$ is a sub-uniformly distributed random variable.
To be clear, we are not claiming which \emph{particular} indices of $\H_1$ will be added to $\S_t$ (indeed, $\mathcal{I}$ is a random set), just that their number exceeds $(1-\delta)|\H_1|$ in expectation.

First we consider the case of Successive Elimination. Let $T_i = \min\{ t \in \mathbb{N}: i \notin \calS_t \}$ be the random number of times arm $i$ is pulled until the last time when it is added to $\calS_t$ (may possibly be infinite).
At round $t$ all arms in $[n] - \calS_t$ are pulled and once an arm is added to $\calS_t$ it will never be pulled again.
Note that $i \in s(k)$ implies $\widehat{\mu}_{i,t} - \phi(t,\delta \tfrac{k}{n}) \geq \mu_0$.
Since for all $i \in \calI$
\begin{align*}
\widehat{\mu}_{i,t} - \phi(t,\delta \tfrac{k}{n}) &\geq \mu_i - \phi(t,\delta \tfrac{k}{n}) - \phi(t,\delta) \\
&\geq \mu_i - 2\phi(t,\delta \tfrac{k}{n})  \\
&= \mu_0 + \Delta_i - 2\phi(t,\delta \tfrac{k}{n})
\end{align*}
we have that $i \in s(k)$ whenever $t \geq \phi^{-1}( \tfrac{\Delta_i}{2}, \delta \tfrac{k}{n})$.
In particular, because $i \in s(1)$ implies $i \in \calS_t$ we have that $T_i \leq \phi^{-1}( \tfrac{\Delta_i}{2}, \delta \tfrac{1}{n})$ for all $i \in \calI$.
But if $\Delta = \min_{i \in \H_1} \Delta_i$ then we also have necessarily that $t \leq \phi^{-1}( \tfrac{\Delta}{2}, \delta \tfrac{|\calI|}{n})$ since at this time, all arms in $\calI$ will be in $s(|\calI|)$, which means $\widehat{k} \geq |\calI|$ and so $\calI \subseteq \calS_t$.
This, of course, implies $T_i \leq \phi^{-1}( \tfrac{\Delta}{2}, \delta \tfrac{|\calI|}{n})$ for all $i \in [n] \setminus \calI$.
Putting these pieces together, we have that
\begin{align*}
\sum_{i=1}^n T_i &\leq \min\Big\{ n \phi^{-1}( \tfrac{\Delta}{2}, \delta \tfrac{|\calI|}{n}), (n-|\calI|) \phi^{-1}( \tfrac{\Delta}{2}, \delta \tfrac{|\calI|}{n}) + \sum_{i \in \calI} \phi^{-1}( \tfrac{\Delta_i}{2}, \delta \tfrac{1}{n}) \Big\} \\
&\leq \min\Big\{ c n \Delta^{-2} \log(\tfrac{n}{|\calI|} \log(\Delta^{-2})/\delta) ,\\
&\hspace{.6in} c (n-|\calI|) \Delta^{-2} \log(\tfrac{n}{|\calI|} \log(\Delta^{-2})/\delta) + \sum_{i \in \calI} c  \Delta_i^{-2} \log(n  \log(\Delta_i^{-2})/\delta) \Big\}.
\end{align*}
Now on event $\mathcal{E}_5$ (which holds with probability at least $1-\delta$) we have
\begin{align*}
|\calI| &\geq (1- \delta - \sqrt{\tfrac{2\delta \log(1/\delta)}{|\H_1|}} - \tfrac{\log(1/\delta)}{3|\H_1|}) |\H_1|
\end{align*}
which implies that for $\eta = \delta + \sqrt{\tfrac{2\delta \log(1/\delta)}{|\H_1|}} + \tfrac{\log(1/\delta)}{3|\H_1|}$ we have
\begin{align*}
\sum_{i=1}^n T_i &\leq \min\Big\{ c n \Delta^{-2} \log(\tfrac{n}{(1-\eta)|\H_1|} \log(\Delta^{-2})/\delta) ,\\
&\hspace{.2in} c (n-(1-\eta)|\H_1|) \Delta^{-2} \log(\tfrac{n}{(1-\eta)|\H_1|} \log(\Delta^{-2})/\delta) + \sum_{i \in \H_1} c  \Delta_i^{-2} \log(n  \log(\Delta_i^{-2})/\delta) \Big\}.
\end{align*}

Now, in the case of Uniform allocation, we never stop pulling arms once they enter $\S_t$. Hence, we need to consider the number of samples needed before all the arms in $\calI\subset \S_t$. By the same reasoning as above, this is bounded by $c n \Delta^{-2} \log(\tfrac{n}{(1-\eta)|\H_1|} \log(\Delta^{-2})/\delta)$.
\end{proof}

The following two theorems provide lower bounds for Uniform allocation.

\begin{theorem}[FDR, TPR]\label{thm:uniform_fdr_lower_bound} Fix $\delta < 1/40$, $\Delta > 0$. and $k < n/2$.
For any $\H_1 \subseteq [n]$ such that $|\H_1|=k$ define an instance $(\{\nu_i\}_{i=1}^n, \mu_0)$ with $\mu_0 = 0$, $\nu_i =\mathcal{N}(\Delta, 1)$ for $i\in \H_1$ , $\nu_i =\mathcal{N}(0,1)$ for $i\in \H_0$.
Any algorithm that samples each arm an equal number of times before outputting a set $\calS \subseteq [n]$ after $\tau$ total samples, and is FDR-$\delta$ and TPR-$\delta,\tau$ on $(\{\nu_i\}_{i=1}^n, \mu_0)$ for all $\H_1 \subseteq [n]$ such that $|\H_1| = k$ simultaneously, must satisfy $t \gtrsim n\Delta^{-2}\log{(n/k)}$.




\end{theorem}
\begin{proof}
The proof is based on the construction of \cite{raginsky2011lower} which states that for any $n > 2k$, there exists a collection  $\mathcal{M}_{n,k}$ of subsets of $[n]$ where a) each $\pi \in \mathcal{M}_{n,k}$ has weight $|\pi|=k$, b) $2k \geq |\pi \triangle \pi'| > k$ for all $\pi \neq \pi' \in \mathcal{M}_{n,k}$, and c) $|\mathcal{M}_{n,k}| \geq (\tfrac{n}{6k})^{k/4}$.
Each $\pi$ gives rise to an instance of the problem $(\{\nu_i\}_{i=1}^n, \mu_0)$ where $\mu_0 = 0$, $\H_1=\pi$ and so $\nu_i = \mathcal{N}(\Delta, 1)$ if $i\in \pi$, otherwise and $\nu_i = \mathcal{N}(\mu_0, 1)$ otherwise. In particular, for each instance $|\H_1| = k$.

Note that if $(i)$, $\frac{|\H_1 \cap \calS_t^c|}{|\H_1|} = 1-\frac{|\H_1 \cap \calS_t|}{|\H_1|} \leq \eta$ and $(ii)$ $\frac{|\H_0 \cap \calS_t|}{|\H_0 \cap \calS_t| + |\H_1 \cap \calS_t|} = \frac{|\H_0 \cap \calS_t|}{|\calS_t|} \leq \eta$ then
\begin{align*}
|\H_1 \triangle \calS_t| &= |\H_1 \cap \calS_t^c| + |\H_1^c \cap \calS_t| \\
&=  |\H_1 \cap \calS_t^c| + |\H_0 \cap \calS_t| \\
&\overset{(i)}{\leq} \eta |\H_1| + |\H_0 \cap \calS_t| \\
&\overset{(ii)}{\leq} \eta|\H_1| + \tfrac{\eta}{1-\eta} |\H_1 \cap \calS_t| \\
&\leq \tfrac{2\eta}{1-\eta} |\H_1|.
\end{align*}
Thus, if $(i)$ and $(ii)$ hold and $\eta<1/5$ then $|\H_1 \triangle \calS_t| \leq k/2$ and so $\H_1 \triangle \calS_t$ can therefore be used as an estimator for any $\H_1 = \pi \in \mathcal{M}_{n,k}$ since $\min_{\pi,\pi' \in \mathcal{M}_{n,k}} |\pi \triangle \pi'| > k$.

By assumption, $\E[\tfrac{|\H_0 \cap \calS_t|}{|\calS_t|}] \leq \delta$ so by Markov's inequality we have $\P( \frac{|\H_0 \cap \calS_t|}{|\calS_t|} \geq 8 \delta) \leq 1/8$.
Likewise, by assumption $\E[\frac{|\H_1 \cap \calS_t^c|}{|\H_1|}] = 1 - \E[\frac{|\H_1 \cap \calS_t|}{|\H_1|}] \leq \delta$, so again $\P( \frac{|\H_1 \cap \calS_t^c|}{|\H_1|} \geq 8 \delta ) \leq 1/8$.
Thus, with probability at least $3/4$, $\frac{|\H_1 \cap \calS_t^c|}{|\H_1|} < 8 \delta$ and $\tfrac{|\H_0 \cap \calS_t|}{|\calS_t|} \leq 8 \delta$.
To apply the above argument, we just need $8 \delta < 1/5$ which holds when $\delta < 1/40$, then $\calS_t$ could predict the correct $\pi \in \mathcal{M}_{n,k}$ with probability at least $1/4$.

We will now use an information theoretic inequality to lower bound the $t$ that would make such an estimator possible.
Let $P_\pi$ be the probability law of sampling $\tau$ samples from each arm under $\pi$.
Then $KL(P_{\pi},P_{\pi'}) \leq \Delta^2 \tau | \pi \triangle \pi' | /2 \leq \Delta^2  \tau k$.
Directly applying Theorem~2.5 of \cite{tsybakov2009introduction} to our $|\mathcal{M}_{n,k}| \geq (\tfrac{n}{6k})^{k/4}$ hypotheses, we have that any estimator has a probability of misidentification of at least
\begin{align*}
\tfrac{1}{2}\left(1-\frac{2\Delta^2  \tau k}{\log|\mathcal{M}_{n,k}|} - \sqrt{\frac{2\Delta^2  \tau k}{\log^2|\mathcal{M}_{n,k}|}}\right) \geq \tfrac{1}{2}\left(1-\frac{8\Delta^2  \tau }{\log(n/6k)} - \sqrt{\frac{32\Delta^2  \tau }{k\log^2(n/k)}}\right)
\end{align*}
which is at least $1/4$ unless $\tau \gtrsim \Delta^{-2} \log(n/k)$.
\end{proof}

\begin{theorem}[FWER, FWPD]\label{thm:uniform_fwer_lower_bound}
Fix $\delta < 3/8$ and $\Delta >0$. For any $\H_1 \subseteq [n]$ consider an instance $(\{\nu_i\}_{i=1}^n, \mu_0)$ such that $\mu_0=-\Delta/2$ and $\nu_i = \mathcal{N}(\mu_i,1)$ where $\mu_i = \Delta/2$ if $i \in \H_1$ and $\mu_i = -\Delta/2$ if $i \in \H_0$.
Any algorithm that samples each arm an equal number of times before outputting a set $\calS \subseteq [n]$ after $\tau$ total samples, and is FWER-$\delta$ and FWPD-$\delta,\tau$ on $(\{\nu_i\}_{i=1}^n, \mu_0)$ for all $\H_1 \subseteq [n]$ simultaneously, must satisfy $\tau \gtrsim n\Delta^{-2} \log(n)$.
\end{theorem}
\begin{proof}
Fix $t \in \mathbb{N}$.
Because the empirical mean is a sufficient statistic for each arm and they are independent, the joint probability distribution of the from the $n$ arms is given by $P_0 := \prod_{i=1}^n \mathcal{N}(\mu_i, 1)$ and the distribution after $t$-pulls on each arm is $P^t_0 := \prod_{i=1}^n \mathcal{N}(\mu_i, 1/t)$

For any $j \in [n]$ define $P_j$ as the joint distribution if the $j$th arm's identity was flipped: if $j \in \H_0$ replace its mean with $1$, if $j \in \H_1$ replace its means with $\H_0$.
Note that $P_j := P_0 \frac{\mathcal{N}(-\mu_i,1)}{\mathcal{N}(\mu_i,1)}$ and that $KL(P_j | P_0) = KL( \mathcal{N}(-\mu_i,1) | \mathcal{N}(\mu_i,1)) = 2 \mu_i^2 = \Delta^2 / 2$ and $KL(P_j^t | P_0^t) = t\Delta^2 / 2$.

Now, because the algorithm was assumed FWER-$\delta$ and FWPD-$\delta,t$ on all instances indexed by $\H_1 \subseteq [n]$, it will be able to distinguish between $\{ P_k \}_{k=0}^n$ using just $t$ per arm with probability at least $1-2\delta \geq 1/4$. The multiple hypothesis testing lower bound of Tsybakov \cite[Theorem 2.5]{tsybakov2009introduction}, implies that the probability of misclassification of any estimator will be at least
\begin{align*}
    \frac{1}{2}( 1- \frac{t \Delta^2}{\log(n)} - \sqrt{\frac{t \Delta^2}{\log^2(n)}})  \geq \frac{1}{4}
\end{align*}
unless $t\gtrsim \Delta^{-2}\log{n}$.

\end{proof}

\section{Technical Lemmas}
\begin{lemma}\label{lem:exponential_tails}
Fix $a \in \R_+^n$ and for $i=1,\dots,n$ let $Z_i$ be independent random variables satisfying $\P( Z_i \geq t ) \leq \exp(-t/a_i)$. Then
\begin{align*}
\P\left( \sum_{i=1}^n (Z_i - a_i) \geq t \right)\leq \exp\left( -\min\{ \tfrac{t}{4||a||_\infty}, \tfrac{t^2}{4||a||_2^2} \} \right)
\end{align*}
and moreover, with probability at least $1-\delta$, $\sum_{i=1}^n Z_i \leq 5 \log(1/\delta) \sum_{i=1}^n a_i$.
\end{lemma}
\begin{proof}
For $i=1,\dots,n$ we have that $Z_i = a_i \log(1/\rho_i)$ are independent random variables satisfying $\P( Z_i \geq t ) \leq \exp(-t/a_i)$, because the $\rho_i$ are independent \emph{sub-uniformly} distributed random variables.
It is straightforward to verify that $\log(\E[\exp(\lambda (Z_i-a_i))]) \leq -a_i \lambda - \log(1-\lambda a_i) \leq \frac{a_i^2 \lambda^2}{2(1-\lambda a_i)}$ for $\lambda \leq 1/||a||_\infty$.
Using the standard Chernoff-bound technique, we have for $\lambda = \min\{ \frac{t}{2||a||_2^2}, \frac{1}{2 ||a||_\infty} \}$ that
\begin{align*}
\P\left( \sum_{i=1}^n (Z_i - a_i) \geq t \right) &\leq \exp\left( -\lambda t + \sum_{i=1}^n \frac{a_i^2 \lambda^2}{2(1-\lambda a_i)}\right)\\
&\leq \exp\left( -\lambda t + \lambda^2 ||a||_2^2 \right) \\
&\leq \exp\left( -\min\{ \tfrac{t}{4||a||_\infty}, \tfrac{t^2}{4||a||_2^2} \} \right) \\
&\leq \exp\left( -\tfrac{1}{4} \min\{ \tfrac{t}{||a||_1}, \tfrac{t^2}{||a||_1^2} \} \right)
\end{align*}
where the last inequality holds by $||a||_\infty \leq ||a||_2 \leq ||a||_1 = \sum_{i=1}^n a_i$.
The result follows from setting the right hand side equal to $\delta$ and solving for $t$.
\end{proof}

\begin{lemma}\label{lem:brownian_bridge}
Fix $\delta \in (0, 1/2]$.
Let $X_i \in [0,1]$ for $i=1,\dots,m$ be independent random variables, each satisfying $\P(X_i \leq s) \leq s$.
Define $A(s) = \sum_{i=1}^m \1\{ X_i \leq s\}$, then for any $c \in (0,1)$
\begin{align*}
\P\left( \exists s \in (0,1]: A(s) > s m + (1+4s) \sqrt{2 \max\{2s, c\} m \log(\tfrac{\log_2(2/c)}{\delta}) } + \tfrac{1+4s}{3} \log(\tfrac{\log_2(2/c)}{\delta}) \right) \leq \delta.
\end{align*}
Moreover,
\begin{align*}
\P\left( \exists s \in (0,1]: A(s) > s m + (1+2s) \sqrt{4 sm \log(\tfrac{2\log_2(2/s)^2}{\delta}) } + \tfrac{1+2s}{3} \log(\tfrac{2\log_2(2/s)^2}{\delta}) \right) \leq \delta.
\end{align*}
Also, recall that by the Dvoretzky-Kiefer-Wolfowitz inequality \cite{massart1990tight} we have $\P\left( \exists s \in (0,1]: A(s) > s m + \sqrt{m \log(1/\delta)/2 }\right) \leq \delta$.
\end{lemma}
\begin{proof}
First note that $M(s) = \frac{A(s) - \E[A(s)]}{1-s}$ is a martingale with respect to the filtration $\mathcal{F}_t = \{ A(s) : s \leq t\}$.
Thus, for $\lambda >0$ we have that $\exp(\lambda M(s) )$ is a non-negative sub-martingale and we can apply Doob's maximal inequality to obtain
\begin{align*}
\P\left( \sup_{s \leq t} \ M(s)  \geq \epsilon/(1-t) \right) &= \P\left( \sup_{s \leq t} \ \exp(\lambda M(s))  \geq \exp(\lambda \epsilon/(1-t)) \right)\\
&\leq \exp(-\lambda \epsilon/(1-t)) \E\left[ \exp(\lambda M(t) ) \right] \\
&= \exp(-\tfrac{\lambda}{1-t} \epsilon ) \E\left[ \exp( \tfrac{\lambda}{1-t} (A(t) - \E[A(s)]) ) \right].
\end{align*}
Observe that for all $t <1$ we have $$\min_\lambda \exp(-\tfrac{\lambda}{1-t} \epsilon ) \E\left[ \exp( \tfrac{\lambda}{1-t} (A(t) - \E[A(s)]) ) \right] = \min_\lambda \exp(-\lambda \epsilon ) \E\left[ \exp( \lambda (A(t) - \E[A(s)]) ) \right].$$
Noting that $A(t)$ is a sum of independent random variables with each in $[0,1]$ and expectation less than $t$ so that $\E[A(t)] \leq m t$, we apply Bernstein's inequality to obtain $\log\E\left[ \exp(\lambda(A(t) - \E[A(t)])) \right] \leq \frac{mt\lambda^2}{2(1-\lambda/3)}$ for $\lambda \in (0, 3)$. Optimizing over $\lambda \in (0,3)$ we have
\begin{align*}
\P\left( \exists s \leq t : A(s) > \E[A(s)] + \tfrac{1-s}{1-t} \sqrt{ 2 mt  \log(1/\delta) } + \tfrac{1-s}{1-t}\log(1/\delta)/3 \right) \leq \delta
\end{align*}
For $k \in \mathbb{N}$ define $T_k = \{ s \in [0,1] : 2^{-k-1} < s \leq 2^{-k} \}$.
Note that $2^{-\lfloor \log_2(2/c) \rfloor} \leq c$.
So for any $k=1,2,\dots,\lfloor \log_2(2/c) \rfloor$, with probability at least $1- \frac{\delta}{\lfloor \log_2(2/c) \rfloor}$ we have that for any $s \in T_k$
\begin{align*}
A(s) &\leq \E[A(s)] + \tfrac{1-s}{1-2^{-k}} \sqrt{ 2 \max\{c,2^{-k}\} m \log(\log_2(2/c)/\delta) } + \tfrac{1-s}{1-2^{-k}} \log( \log_2(2/c)/\delta)/3.
\end{align*}
1Note that cases $s \in T_k$ for $k > \lfloor \log_2(2/c) \rfloor$ are handled by $k=\lfloor \log_2(2/c) \rfloor$.
For any $k \geq 1$ and $s \in T_k$ we have $s \geq 2^{-k-1}$ so that $1+4s \geq \tfrac{1}{1-\min\{2s,2^{-1}\}} \geq \tfrac{1-s}{1-2^{-k}}$ and $2s \geq 2^{-k}$.
Thus,
\begin{align*}
\hspace{.05in}&\hspace{-.05in}\P\left( \exists s \in T_k : A(s) > \E[A(s)] + (1+4s)  \sqrt{ 2 \max\{c,2s\} m \log(\log_2(2/c)/\delta) } + (1+4s) \log( \log_2(2/c)/\delta)/3  \right)\\
&\leq\P\left( \exists s \in T_k : A(s) > \E[A(s)] + \tfrac{1-s}{1-2^{-k}} \sqrt{ 2 \max\{c, 2^{-k}\} m \log(\log_2(2/c)/\delta) } + \tfrac{1-s}{1-2^{-k}} \log( \log_2(2/c)/\delta)/3  \right)\\
&\leq \delta\frac{1}{\log_2(2/c)} \leq \delta \frac{1}{\lfloor \log_2(2/c) \rfloor}.
\end{align*}
Union bounding over $k=2,\dots,\lfloor \log_2(2/c) \rfloor$ handles $\cup_{k = 2}^\infty T_k = (0,1/4]$.
To handle $s \in (1/4,1]$, we note that
\begin{align*}
\hspace{.05in}&\hspace{-.05in}\P\left( \exists s \in (1/4,1] : A(s) > \E[A(s)] + (1+4s)  \sqrt{ 2 \max\{c,2s\} m \log(\log_2(2/c)/\delta) } + (1+4s) \log( \log_2(2/c)/\delta)/3  \right)\\
&\leq\P\left( \exists s \in T_k : A(s) > \E[A(s)] + \sqrt{ m \log(\log_2(2/c)/\delta) } \right)\\
&\leq \left( \delta\frac{1}{\log_2(2/c)}\right)^2 \leq \delta \frac{1}{\lfloor \log_2(2/c) \rfloor}
\end{align*}
where the second to last inequality holds by the DKW inequality \cite{massart1990tight}.

On the other hand, for any $k \geq 1$ and $s \in T_k$ we have $s \geq 2^{-k-1}$ so that $1+2s \geq \tfrac{1-s}{1-\min\{2s,2^{-1}\}} \geq \tfrac{1-s}{1-2^{-k}}$ and $2^{-k-1} \leq s \leq 2^{-k}$.
\begin{align*}
\hspace{.05in}&\hspace{-.05in}\P\left( \exists s \in T_k : A(s) > \E[A(s)] + (1+2s) \sqrt{ 4 s m \log(2\log_2(\tfrac{2}{s})^2/\delta) } + (1+2s) \log( 2 \log_2(\tfrac{2}{s})^2/\delta)/3  \right)\\
&\leq\P\left( \exists s \in T_k : A(s) > \E[A(s)] + \tfrac{1-s}{1-2^{-k}} \sqrt{ 2 \cdot 2^{-k} m \log(2(k+1)^2/\delta) } + \tfrac{1-s}{1-2^{-k}} \log( 2 (k+1)^2/\delta)/3  \right)\\
&\leq \delta\frac{1}{2 (k+1)^2}.
\end{align*}
Union bounding over all $k \geq 0$ and noting that $\sum_{k=0}^\infty \frac{1}{2(k+1)^2} \leq 1$ completes the proof since $\cup_{k \geq 0} T_k = (0,1]$
\end{proof}


\end{document}